\documentclass{article}
\usepackage{iclr2027_conference,times}

\usepackage[dvipsnames,table,xcdraw]{xcolor}

\usepackage{amsmath,amsfonts,bm}

\def\eqref#1{equation~\ref{#1}}
\def\1{\bm{1}}

\DeclareMathAlphabet{\mathsfit}{\encodingdefault}{\sfdefault}{m}{sl}
\SetMathAlphabet{\mathsfit}{bold}{\encodingdefault}{\sfdefault}{bx}{n}

\colorlet{LightRubineRed}{RubineRed!60!white}
\usepackage[colorlinks=true, linkcolor=LightRubineRed, citecolor=LimeGreen, urlcolor=Cyan]{hyperref}
\usepackage{url}
\usepackage{subcaption}
\usepackage{graphicx}
\usepackage{booktabs}
\usepackage{makecell}
\usepackage{algorithm}
\usepackage{algpseudocode}
\usepackage{xspace}
\usepackage{soul}

\usepackage{amsthm}
\newtheorem{theorem}{Theorem}
\newtheorem{proposition}[theorem]{Proposition}

\title{\centering Variational Atlas Autoencoders\\ and Self-Supervised Chart Selection}

\author{\parbox{\dimexpr\textwidth-2\tabcolsep}{\centering
    Arno Coomans$^{1,2}$ \quad
    Merlin Nimier-David$^{1}$ \quad
    Eugene d'Eon$^{1}$ \quad
    Boris Bonev$^{1}$ \\[0.6em]
    \textnormal{$^{1}$NVIDIA \quad $^{2}$Graz University of Technology, Austria}
}}

\newcommand{\methodnamesingular}{Variational Atlas Autoencoder\xspace}
\newcommand{\methodname}{Variational Atlas Autoencoders\xspace}

\newcommand{\x}{\mathbf{x}}
\newcommand{\z}{\mathbf{z}}

\iclrfinalcopy

\begin{document}

\maketitle

\begin{abstract}

Complex, high-dimensional data lie on manifolds that a single global latent coordinate system cannot represent without distortion.
We introduce \methodname, a mixture-of-charts framework that learns an atlas of local latent representations, each responsible for one region of the data manifold.
The central difficulty in learning an atlas is deciding which chart should be responsible for which data point without labels.
We show that this decision need not be a heuristic. By treating the chart index as a latent variable, the evidence lower bound yields the training objective and casts chart assignment as posterior inference, valid for any routing distribution.
Within this framework, we propose \emph{latent-evidence routing}, which assigns a point to a chart in proportion to the evidence of its encoding under that chart's latent prior.
This has a simple interpretation: the closer a point lands to the center of a chart's latent space, and the more confidently it is placed there, the more likely it is to belong to that chart.
Chart assignment thus emerges as posterior inference, depends on the encoder alone, and is fully self-supervised, so the charts discover structure in the data on their own, separating classes and subclasses such as distinct types of shoes or faces without access to labels.
We further recover an atlas from the trained charts by identifying their overlaps and transition maps, enabling continuous interpolation within and across charts that stays close to the data manifold rather than traversing unsupported regions of latent space.
Compared to capacity-matched baselines, \methodname yields better-structured latent spaces and higher sample quality on image datasets including CelebA, while sampling requires evaluating only a single chart's decoder.
\end{abstract}

\section{Introduction}
\label{sec:intro}

Latent spaces learned by autoencoders underpin much of modern generative modeling, from latent diffusion~\citep{rombach2022high, vahdat2021score} to manifold-learning flows~\citep{brehmer2020flows}.
These models are only as good as the learned coordinate system they are expressed in, and those coordinates almost always form a single global representation.  Forcing all data into one representation is a design choice inherited from the autoencoder, not a property of the data.

Under the manifold hypothesis, data concentrate near a $k$-dimensional manifold $\mathcal{M} \subset \mathbb{R}^d$ with $k \ll d$~\citep{fefferman2016testing, brown2022verifying}, and an autoencoder can be viewed as an approximate \emph{chart} of $\mathcal{M}$: the encoder assigns coordinates to each data point, and the decoder maps them back.
However, a single chart cannot cover every manifold without tearing or distortion. For example, a sphere cannot be flattened onto a plane, and data made of separate clusters, such as distinct classes, cannot be covered by one continuous latent space without filling the gaps between them.
Autoencoders thus resort to non-smooth latent spaces whose interpolants decode to implausible samples~\citep{ArvanitidisHH18, interp_oring21a}.
This is unavoidable for disconnected data: the image of a connected latent space is connected, so it must cross the gaps between components.

Differential geometry resolves this with an \emph{atlas}: charts whose domains cover $\mathcal{M}$, overlap, and are related by transition maps.
Chart autoencoders~\citep{schonsheck2019chart} and mixtures of VAEs~\citep{ye2021deep, alberti2024manifold} move in this direction by giving each component its own encoder--decoder pair, but leave two problems open.

\begin{figure}[t]
    \centering
    \includegraphics[width=\linewidth]{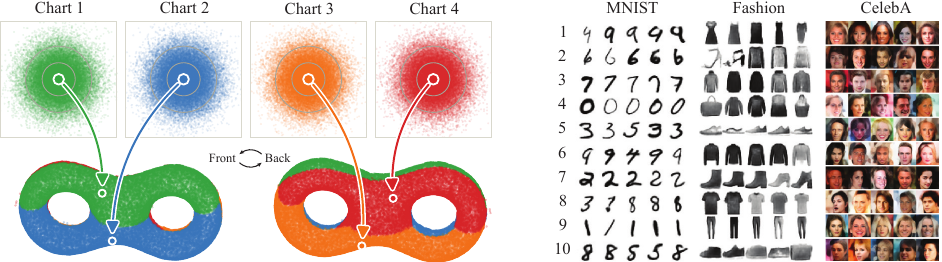}
    \caption{
        \methodname specialize to meaningful subsets of the data manifold with no label supervision.
        \textbf{Left}: a genus-2 surface covered by four $2$D charts, colored by the chart claiming each region, shows that multiple charts avoid the distortion of a single global 2D latent space.
        \textbf{Right}: generated samples per chart on three image datasets.
        On MNIST~\citep{lecun1998gradient} and Fashion-MNIST~\citep{fashionMNIST} ($16$-D latents), a chart typically covers one class or a related pair; on CelebA~\citep{liu2015faceattributes} ($128$-D latents), charts split by facial features, illumination, and background.
    }
    \label{fig:teaser}
\end{figure}

Without labels, the model must decide the \textit{responsibility} of each chart for each data point on the manifold.
We make this decision part of the probabilistic model by treating the chart index as a latent variable.
Responsibility then enters the ELBO as the variational posterior $q(c|x)$, which we call the routing function, and since the bound holds for any choice of $q(c|x)$, the question becomes which router trains best.
Routing by each chart's 
ELBO~\citep{alberti2024manifold} ties assignment to how well a chart already reconstructs a point, which can favour charts that improve early.
On the other hand, separate latent spaces use disjoint coordinate systems, so interpolation is undefined across charts, limiting the \textit{connectivity} of the model.
Without a way to transition between charts, the model cannot interpolate between arbitrary points on the data manifold, hindering downstream applications expressed in this coordinate system, such as latent-space diffusion.

\paragraph{Contributions.}
\methodname address both problems.
(i)~We formulate a variational mixture of charts, each a complete VAE, in which the training objective follows from the ELBO.
(ii)~We propose \emph{latent-evidence routing}, which assigns a point to a chart in proportion to the evidence of its encoding under that chart's prior: the closer the encoding lies to the chart's centre, and the more confident it is, the higher its responsibility.
It depends on the encoders alone, meaning it does not require reconstructions for routing.
It discovers classes and subclasses without labels, such as sandals and heeled shoes in Fashion-MNIST.
(iii)~We turn the trained charts into an atlas by fitting overlaps (\emph{portals}) where the router ties and building a chart transition graph, enabling interpolation within and across charts that stays close to the data manifold.
(iv)~On MNIST, Fashion-MNIST, and CelebA, our method achieves the best sample FID among parameter-matched VAE and VaDE baselines, while sampling requires only a single (smaller) decoder pass.

\section{Related Work}

\paragraph{Manifold learning.}
The manifold hypothesis posits that high-dimensional data concentrate near a manifold of much lower intrinsic dimension.
\citet{fefferman2016testing} provide a rigorous framework for testing this hypothesis for sufficiently smooth manifolds.
For image data, \citet{brown2022verifying} provide empirical evidence for a more general union-of-manifolds hypothesis: data concentrate near disjoint manifolds with potentially varying intrinsic dimensions.
They further show that incorporating this structure into model design improves classification and generation.

These observations suggest that using a single continuous latent space is generally insufficient to represent the data manifold.
\citet{tanielian2020learning} quantify this phenomenon for GANs: covering separated data components involves generating low-quality samples in the gaps between them.
\citet{batson2021topological} demonstrate the same issue in the context of anomaly detection, giving low-dimensional examples where topological mismatch leads to reconstruction errors for otherwise ordinary data points.
Our use of multiple charts allows separate representations of disconnected components and local coordinates within components whose topology precludes a single global chart, with transition maps relating coordinates on overlaps.

\paragraph{Latent geometry and interpolation.}
Geometry also matters for density estimation and interpolation.
Manifold-learning flows \citep{brehmer2020flows} jointly learn an embedded manifold and a tractable density on it, while \citet{arvanitidis2020geometrically} equip the ambient space with a Riemannian metric and pull it back through the generator to construct more meaningful interpolation paths.
Together, these works motivate separating local parameterization from global structure.
Although we do not aim to generate the shortest interpolation paths, our charts are well-suited to both intra-chart and cross-chart interpolation along straight lines thanks to their local coverage of the data manifold.

\paragraph{Multi-chart and mixture models.}
Mixtures of local linear latent-variable models \citep{ghahramani1996em,tipping1999mixtures} are the classical precursor to our model, with mixture responsibilities in the role of the routing posterior.
\citet{teh2003automatic} and \citet{brand2003charting} align such local models into one global embedding, whereas we keep the local coordinates and learn the transition maps between them (\autoref{eq:transition}).
 
In the nonlinear (deep) setting, variational deep embeddings (VaDE) \citep{jiang2016variational} place a Gaussian mixture prior on the latent space of a single VAE and read the cluster assignment from the mixture responsibilities.
Since all components share the encoder and decoder, this is still a single chart, which cannot cover most non-trivial data manifolds (\autoref{app:onechart}).
Mixtures of VAEs with a separate encoder--decoder pair per component \citep{ye2021deep,alberti2024manifold} remove this restriction and are the closest generative models to ours.
To route data to a specific chart, \citet{alberti2024manifold} use the per-component ELBO.
We show that this couples assignment to reconstruction quality and starves charts that start out worse, and we replace it with latent-evidence routing, which depends on the encoder alone.
Neither work constructs overlaps or transition maps between components.

The chart auto-encoder (CAE) of \citet{schonsheck2019chart} uses several deterministic encoder--decoder pairs and shows that a single simply connected latent space cannot faithfully represent a manifold with non-contractible topology.
Charts are assigned by a separate network trained toward a softmax over per-chart reconstruction errors, and transition maps compose one chart's decoder with another's encoder, but the overlaps on which they should hold are never constructed.
\citet{stolberglarsen2021atlas} unify such multi-chart models as \emph{atlas generative models} and interpolate across charts on a graph over encoded data points.
Two charts are linked only at individual data points whose routing exceeds a threshold in both, by connecting that point's two encodings.
We differ in three respects: our chart assignment follows from the ELBO (\autoref{eq:multichart-elbo}) and is computed from the chart encoders rather than a separate network; overlaps are portals, distributions over a chart's latent space located where the router ties, rather than individual data points; and each portal comes with an explicit transition map, so charts can be crossed from anywhere within it (\autoref{sec:portals}).

\section{Methodology}
\label{sec:methodology}

The overall goal of a generative model is to represent a data distribution $p(\x)$.
In a latent-variable model, the distribution is often represented using a single continuous latent space with a simple prior, typically a multivariate standard normal distribution $\z \sim \mathcal{N}(\mathbf{0},\mathbf{I})$.
The encoder must therefore represent all samples $\x$ from the data distribution within a single latent space, even when the data lie on a highly nonlinear or disconnected manifold.
Our proposed solution is to model the data distribution using multiple latent representations. Each latent component can specialize in a particular subset, mode, or locally coherent region of the data manifold.

\begin{figure}[t]
    \centering
    \includegraphics[width=\linewidth]{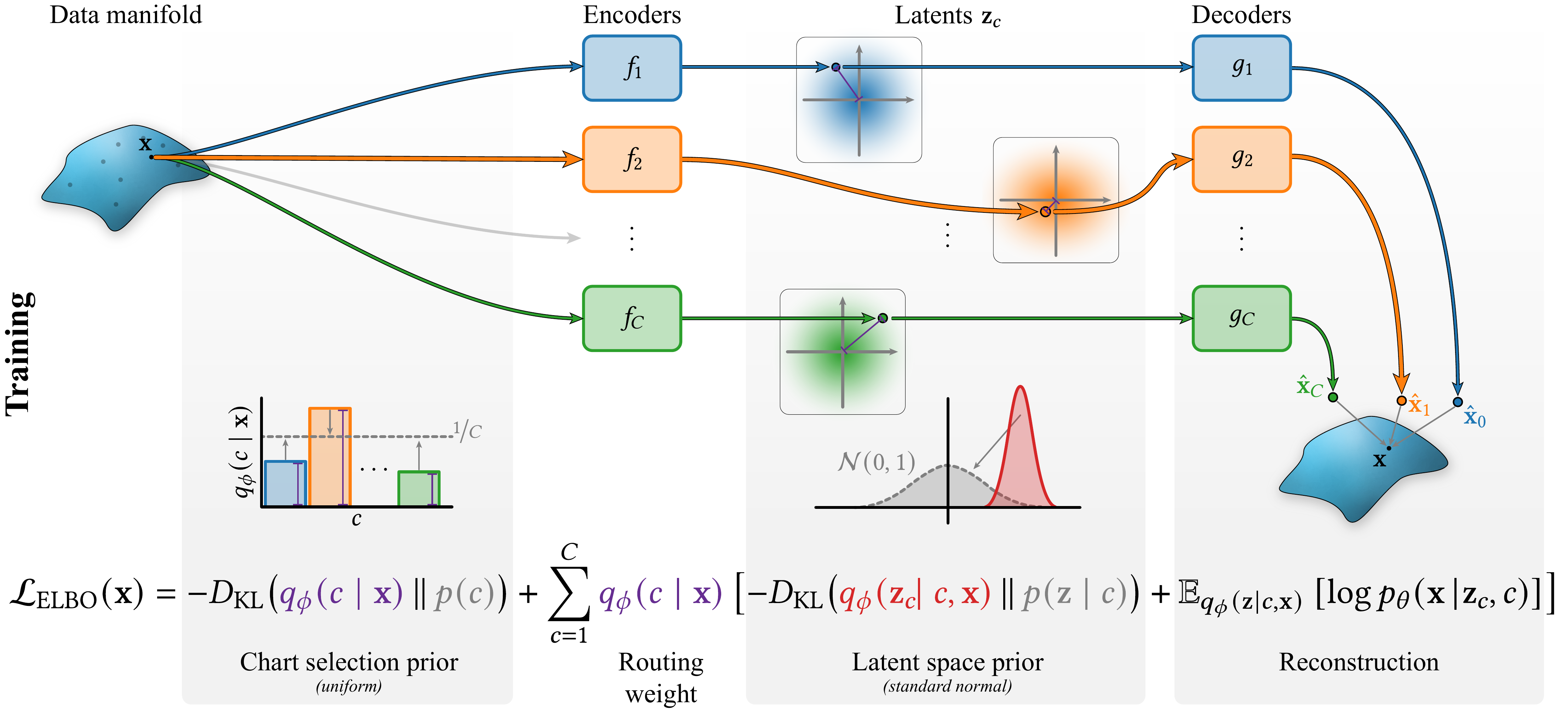}
    \hfill\rule{0.985 \textwidth}{0.4pt}\hfill
    \vspace{0.4em}
    \includegraphics[width=\linewidth]{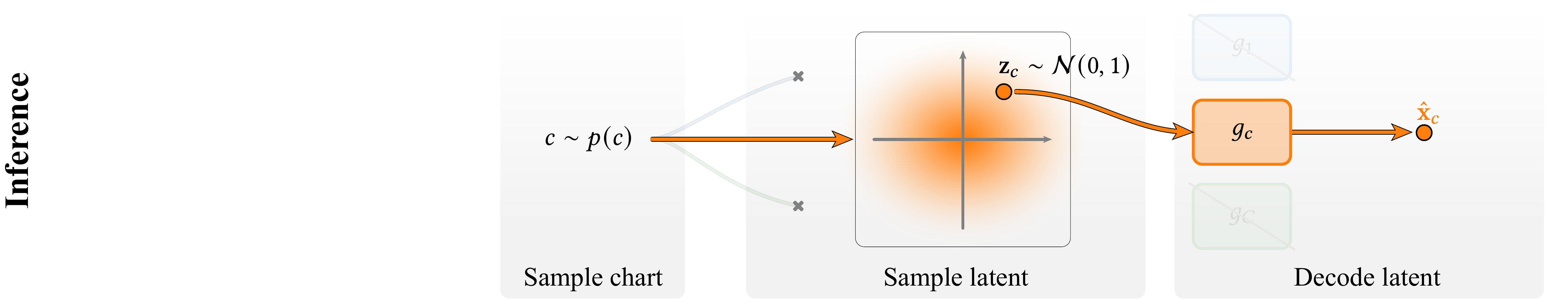}
    \caption{
        \textbf{Top}: we train $C$ encoder--decoder pairs to model the data manifold using independent charts.
        Our training objective, the ELBO, decomposes into per-chart losses weighted by a routing probability.
        The KL divergence on the chart selection prior encourages all charts to be used with equal probability.
        \textbf{Bottom}: sampling from a chart is then as simple as selecting a chart $c$ uniformly at random, sampling a latent $\mathbf{z}_c$ from the standard normal prior, and decoding it.
        Moment matching can optionally be added, see \autoref{sec:sampling}.
    }
    \label{fig:method}
\end{figure}

\subsection{Variational Chart Selection}
\label{sec:variational-chart-selection}

To model $p_\theta(\x)$ using multiple charts, we introduce a discrete chart variable $c\in\{1,\ldots,C\}$. As a modeling assumption, each chart is assigned equal prior probability, so that $p(c)=1/C$. Each chart is further equipped with a continuous $L$-dimensional latent variable $\z_c\in\mathbb{R}^L$ with conditional prior $p(\z_c\mid c)=\mathcal{N}(\mathbf{0},\mathbf{I})$. The resulting modeled data distribution is then given by:
\begin{equation} \label{eq:data-distr-charts}
    p_\theta(\x)
    =
    \sum_{c=1}^{C} p(c)
    \int
    p_\theta(\x\mid\z_c,c)\,
    p(\z_c\mid c)\,
    \mathrm{d}\z_c .
\end{equation}
Given an observation $\x$, inference requires the posterior distribution over the chart assignment $p_\theta(c\mid\x)$ and the chart-specific latent posterior $p_\theta(\z_c\mid c,\x)$. Since these posterior distributions are generally intractable, we introduce the variational approximations $q_\phi(c\mid\x)$ and $q_\phi(\z_c\mid c,\x)$. Together, they define the approximate posterior
\begin{equation} \label{eq:approximate-posterior}
    q_\phi(c,\z_c\mid\x)
    =
    q_\phi(c\mid\x)\,
    q_\phi(\z_c\mid c,\x),
\end{equation}
which approximates the true posterior $p_\theta(c,\z_c\mid\x)$.
Each chart $c$ is realized by an encoder $f_c(\x) = (\boldsymbol{\mu}_c(\x), \boldsymbol{\Sigma}_c(\x))$, which parameterizes the diagonal Gaussian $q_\phi(\z_c\mid c,\x) = \mathcal{N}(\boldsymbol{\mu}_c(\x), \boldsymbol{\Sigma}_c(\x))$, and a decoder $g_c$, whose output $\hat{\x} = g_c(\z_c)$ is the mean of $p_\theta(\x\mid\z_c,c)$.
Throughout, $\phi$ and $\theta$ denote the parameters of all encoders and decoders, respectively.
Direct evaluation and optimization of the marginal log-likelihood in \autoref{eq:data-distr-charts} are generally intractable because the integral over the continuous latent variables cannot usually be computed in closed form. As in the single-chart VAE setting~\citep{kingma2013auto}, we therefore optimize a variational lower bound on $\log p_\theta(\x)$,
\begin{equation}
    \log p_\theta(\x)
    \geq
    \mathcal{L}_{\mathrm{ELBO}}(\x)
    =
    \mathbb{E}_{q_\phi(c,\z_c\mid\x)}
    \left[
        \log
        \frac{
            p_\theta(\x,\z_c,c)
        }{
            q_\phi(c,\z_c\mid\x)
        }
    \right].
\end{equation}
Relating the lower bound to the exact posterior and expressing terms as KL divergences, we obtain
\begin{align}
\label{eq:multichart-elbo}
\mathcal{L}_{\mathrm{ELBO}}(\mathbf{x})
=
& -D_{\mathrm{KL}}\!\left(
    q_\phi(c\mid\mathbf{x})
    \,\Vert\,
    p(c)
\right) \\
& + \sum_{c=1}^{C}
    q_\phi(c\mid\mathbf{x}) \left[
        \mathbb{E}_{q_\phi(\mathbf{z}_c \mid c,\mathbf{x})}
        \left[
            \log p_\theta(\mathbf{x} \mid \mathbf{z}_c, c)
        \right]
        -
        D_{\mathrm{KL}}\!\left(
            q_\phi(\mathbf{z}_c\mid c,\mathbf{x})
            \,\Vert\,
            p(\mathbf{z}_c\mid c)
        \right)
    \right]. \nonumber
\end{align}
Writing the chart router as an explicit finite sum makes it clear that the per-chart ELBO
\begin{equation}
\label{eq:per-chart-elbo}
\mathcal{L}_c(\x) =
    \mathbb{E}_{q_\phi(\z_c \mid c, \x)}\!\left[
        \log p_\theta(\x \mid \z_c, c)
    \right]
    - D_{\mathrm{KL}}\!\left(
        q_\phi(\z_c \mid c, \x) \,\|\, p(\z_c \mid c)
    \right)
\end{equation}
is simply weighted by each chart's routing probability $q_\phi(c\mid\x)$, and the prior routing term
$D_{\mathrm{KL}}(q_\phi(c\mid\x)\,\Vert\,p(c))$ regularizes the routing function.
The role of each term is visualized in \autoref{fig:method}.
In practice, we minimize the weighted loss $\ell(\x)$ of \autoref{eq:weighted-loss}, in which $\mathcal{L}_{\mathrm{recon}}(\x, g_c(\z_c))$ is either the reconstruction loss $-\log p_\theta(\x\mid\z_c,c)$ or a perceptual loss, and the two KL terms are weighted by hyperparameters $\beta$ and $\gamma$ (\autoref{app:full-derivation}, \autoref{app:experimental-setup}).

\begin{figure}[t]
    \centering
    \includegraphics[width=\linewidth]{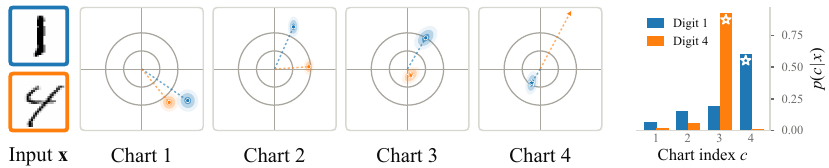}
    \caption{Illustration of latent-evidence routing. Two inputs (left) are encoded by every chart; the evidence of each encoded point under its chart's latent prior induces a distribution over charts, which is used for routing. For the histogram (right), the temperature parameter was artificially raised to $T=10$ to make the distribution easier to visualize.}
    \label{fig:routing-function}
\end{figure}

The bound in \autoref{eq:multichart-elbo} holds for \emph{any} routing distribution $q_\phi(c \mid \mathbf{x})$, which leaves the choice of router as a design decision.

\paragraph{ELBO routing.}
A natural choice is to route by the evidence,
$
q_\phi(c \mid \mathbf{x})
    \propto p(c)\exp\!\left(
        \mathcal{L}_c(\mathbf{x}) / T
    \right)
$,
where $T$ is a temperature parameter that anneals the chart assignment.
Substituting this router into \autoref{eq:multichart-elbo} collapses the bound into
$
\mathcal{L}_\mathrm{ELBO}(\mathbf{x})
    = T \log \sum_c p(c) \exp\!\left(
        \mathcal{L}_c(\mathbf{x}) / T
    \right)
$,
which is tight at $T=1$.
We refer to this variant as \emph{ELBO routing}, following related work~\citep{alberti2024manifold}.

\paragraph{Latent-evidence routing.}
Instead, routing can be based on the latent prior alone, without involving the decoder.
We propose \emph{latent-evidence routing}, guided by the intuition that a chart's routing probability should be proportional to the likelihood of the encoded point under that chart's Gaussian prior.
All variants below take the form $q_\phi(c \mid \mathbf{x}) \propto p(c)\exp\!\left(s_c(\mathbf{x})/T\right)$, a softmax over a per-chart score $s_c$ with temperature $T$, and differ only in how $s_c$ accounts for the posterior uncertainty $\boldsymbol{\Sigma}_c(\mathbf{x}) = \operatorname{diag}(\sigma_{c,1}^2, \dots, \sigma_{c,L}^2)$. \autoref{fig:routing-function} provides a visual overview of latent-evidence routing.

The simplest score evaluates the prior at the posterior mean, $s_c(\mathbf{x}) = \log p(\mathbf{z}_c = \boldsymbol{\mu}_c(\mathbf{x}) \mid c)$, ignoring the posterior uncertainty entirely. We refer to this as \emph{$\mu$-based routing}.
Alternatively, the \emph{expected log-prior} score averages the log-prior over the posterior,
\begin{align}
    s_c(\mathbf{x})
    = \mathbb{E}_{q_\phi(\mathbf{z}_c \mid c, \mathbf{x})}\!\left[\log p(\mathbf{z}_c \mid c)\right]
    = -\frac{1}{2}\left(\|\boldsymbol{\mu}_c\|^2 + \operatorname{tr}\boldsymbol{\Sigma}_c\right) + \mathrm{const}.
    \label{eq:expected-log-prior-routing}
\end{align}
A single-sample Monte Carlo estimate of this expectation recovers the \emph{resampling-based routing}, at the cost of additional variance in the routing weights.

Finally, we propose \emph{prior-marginal routing}, which scores each chart by the prior mass its posterior covers, i.e.\ the evidence of the posterior under the chart's latent prior,
\begin{align}
    s_c(\mathbf{x})
    = \log \int q_\phi(\mathbf{z}_c \mid c, \mathbf{x})\, p(\mathbf{z}_c \mid c)\, \mathrm{d}\mathbf{z}_c
    = -\frac{1}{2}\left(\|\boldsymbol{\mu}_c\|^2_{(\mathbf{I}+\boldsymbol{\Sigma}_c)^{-1}} + \log\det(\mathbf{I} + \boldsymbol{\Sigma}_c)\right) + \mathrm{const},
    \label{eq:prior-marginal-routing}
\end{align}
where $\|\boldsymbol{\mu}_c\|^2_{(\mathbf{I}+\boldsymbol{\Sigma}_c)^{-1}} = \boldsymbol{\mu}_c^\top (\mathbf{I}+\boldsymbol{\Sigma}_c)^{-1} \boldsymbol{\mu}_c$.
This is the marginal likelihood $\mathcal{N}(\boldsymbol{\mu}_c(\mathbf{x});\, \mathbf{0},\, \mathbf{I} + \boldsymbol{\Sigma}_c(\mathbf{x}))$ of the encoder output regarded as a noisy observation of a latent $\mathbf{z}_c \sim p(\mathbf{z}_c \mid c)$ with observation noise $\boldsymbol{\Sigma}_c(\mathbf{x})$.
All three scores are available in closed form for diagonal Gaussian posteriors.
Whereas ELBO routing measures the evidence of $\mathbf{x}$ in data space and therefore depends on the decoder, $s_c$ measures the evidence of the encoded point in latent space and depends on the encoder only.
It formalizes the latent-centrality intuition of \autoref{sec:intro}: a point belongs to the chart whose posterior places it closest to the center of the chart's latent domain.

The three scores differ only in how they treat uncertainty, and in the limit $\boldsymbol{\Sigma}_c \to \mathbf{0}$, all of them coincide with $\mu$-based routing.
The expected log-prior (\autoref{eq:expected-log-prior-routing}) is a lower bound on the prior marginal by Jensen's inequality, but it penalizes the full Euclidean distance $\|\boldsymbol{\mu}_c\|^2$ regardless of the posterior spread and merely adds a variance penalty $\operatorname{tr}\boldsymbol{\Sigma}_c$. The prior marginal of~\autoref{eq:prior-marginal-routing} instead measures this distance in the Mahalanobis norm of $\mathbf{I} + \boldsymbol{\Sigma}_c$, so that uncertain encodings are penalized less.

\paragraph{Impact on training.}
The choice of router also affects training: since \autoref{eq:multichart-elbo} is a valid lower bound for any $q_\phi(c \mid \mathbf{x})$, we substitute the router in and maximize $\mathcal{L}_\mathrm{ELBO}$ directly, so the encoders receive gradients through both the per-chart bounds $\mathcal{L}_c$ and the routing weights.
Under ELBO routing, a chart's weight is determined by how well it already reconstructs $\mathbf{x}$. Charts that start out worse receive smaller weights, hence smaller gradients, and fall further behind.
Latent-evidence routing removes this feedback loop by assigning points according to where they land in latent space rather than how well they are reconstructed, and the KL term in \autoref{eq:multichart-elbo} keeps the routing weights from collapsing onto a single chart.
We adopt prior-marginal routing (\autoref{eq:prior-marginal-routing}) as our default due to its robustness and report ablation results for all variants in \autoref{tab:routing_ablation_table} in the appendix.

\subsection{Sampling}
\label{sec:sampling}

\paragraph{Moment matching (optional).}
After training, the aggregate posterior $q_\phi(\mathbf{z}_c \mid c) = \mathbb{E}_{p(\mathbf{x})}\, q_\phi(\mathbf{z}_c \mid c, \mathbf{x})$ of a chart need not coincide exactly with its prior $\mathcal{N}(\mathbf{0}, \mathbf{I})$, and sampling from regions of the prior that carry no posterior mass produces poor samples~\citep{hoffman2016elbo}.
We therefore optionally replace the prior by the moment-matched aggregate posterior $\tilde{p}(\mathbf{z}_c \mid c) := \mathcal{N}(\mathbf{m}_c, \mathbf{S}_c)$, whose mean and covariance are those of $q_\phi(\mathbf{z}_c \mid c)$ estimated once over the training set. Without moment matching, we use $\tilde{p} = p$.
The chart prior is treated likewise: the uniform $p(c) = 1/C$ is replaced by the empirical routing mass $\hat{p}(c) = \mathbb{E}_{p(\mathbf{x})}\, q_\phi(c \mid \mathbf{x})$, so that charts claiming more of the training data are sampled more often.
Both corrections are cheap, require no retraining, and are applied in all sampling experiments unless otherwise stated.

\paragraph{Unconditional sampling.}
Generation follows Algorithm~\ref{alg:sampling}: draw a chart $c \sim \hat{p}(c)$ and a latent $\mathbf{z}_c \sim \tilde{p}(\mathbf{z}_c \mid c)$, and decode $\hat{\mathbf{x}} = g_c(\mathbf{z}_c)$.
Only the selected decoder is evaluated, so the sampling cost is that of a single VAE with the same backbone.
Sampling within one chart yields samples from the corresponding region of the data manifold (\autoref{fig:samples}), which makes the chart index an interpretable, label-free handle on the generated content.

\paragraph{Encoding.}
In contrast, encoding a data point requires all $C$ encoders to evaluate the routing weights $q_\phi(c \mid \mathbf{x})$.
For applications where encoding cost matters, one may train a lightweight classifier to predict the routing distribution so that only the relevant encoder is evaluated.

\subsection{Overlaps, Transition Maps, and Portals}
\label{sec:portals}
An atlas on a manifold is a family of charts whose domains cover the manifold and whose \emph{transition maps} are homeomorphisms on the overlaps (\autoref{app:manifolds}).
The model described in \autoref{sec:variational-chart-selection} supplies coordinate maps, but the routing posterior $q_\phi(c\mid\mathbf{x})$ induces a \emph{partition} of the data, and a partition has no overlaps.
Here, we describe how overlaps and transition maps between charts can be constructed to recover an atlas and thereby enable movement from any point on the manifold to any other.

\paragraph{Transition map.}
The decoder $g_c$ and the encoder mean $\boldsymbol{\mu}_c$ play the roles of the parametrization and the coordinate map of chart $c$ (\autoref{app:autoencoders}), so the transition map from chart $a$ to chart $b$ and its round trip are
\begin{equation}
    \psi_{ab}(\mathbf{z}) = \boldsymbol{\mu}_b\bigl(g_a(\mathbf{z})\bigr), \qquad R_{ab} = \psi_{ba}\circ\psi_{ab}.
    \label{eq:transition}
\end{equation}
The atlas condition on the overlap is $R_{ab} = \mathrm{id}$. We measure latent distances in the Mahalanobis metric of the sampling distribution $\tilde{p}(\mathbf{z}_c \mid c) = \mathcal{N}(\mathbf{m}_c, \mathbf{S}_c)$ of \autoref{sec:sampling}, $\|\mathbf{u}\|_c^2 = \tfrac{1}{L}\,\mathbf{u}^\top\mathbf{S}_c^{-1}\mathbf{u}$, normalised so that a typical latent of chart $c$ has $\|\boldsymbol{\mu}_c(\mathbf{x}) - \mathbf{m}_c\|_c^2\approx 1$.

\paragraph{Portals as variational overlaps.}
An overlap is a region of the data manifold that \emph{both} charts claim, and the router $q_\phi(c\mid\mathbf{x})$ is the model's own statement of which chart claims a point.
We therefore represent the overlap of the ordered pair $(a,b)$ by a Gaussian $q_{ab}(\mathbf{z}) = \mathcal{N}(\mathbf{m}_{ab}, \mathbf{S}_{ab})$ on the latent space of chart $a$, with $\mathbf{S}_{ab}$ parameterized by its Cholesky factor, fitted to lie where the router assigns equal probability to $a$ and $b$ (a tie).
We refer to these overlaps as ``portals'', as they will be used to transition across charts.
The encoders' and decoders' weights are kept frozen while optimizing portal locations.
Writing $\mathbf{x} = g_a(\mathbf{z})$, $(R_{ab})_\# q_{ab}$ for the pushforward of $q_{ab}$ under the round trip $R_{ab}$,
and $\mathrm{mm}(\cdot)$ for the Gaussian with the same mean and covariance as its argument, $q_{ab}$ minimizes
\begin{equation}
  \mathcal{L}_\text{portal}(q_{ab}) =
  -\mathbb{E}_{q_{ab}}\!\left[\tfrac12\bigl(\log q_\phi(a\mid\mathbf{x}) + \log q_\phi(b\mid\mathbf{x})\bigr)\right]
  + \lambda\,D_{\mathrm{KL}}\!\Bigl(\mathrm{mm}\bigl((R_{ab})_\# q_{ab}\bigr)\,\Big\|\,q_{ab}\Bigr).
  \label{eq:portal-objective}
\end{equation}
The first term is minimized at an exact tie $q_\phi(a\mid\mathbf{x}) = q_\phi(b\mid\mathbf{x}) = \tfrac12$, since the two probabilities sum to at most one, and on its own already locates the overlap. The second term favours the part of the tie set on which the round trip preserves $q_{ab}$ as a distribution. This on its own does not enforce pointwise invertibility of $\psi_{ab}$, but the defect can be repaired by an optional post-processing step proposed in \autoref{app:roundtrip_error}.
Finally, expectations are estimated by reparameterized Monte Carlo, and each fit is initialized at the data point where the router most nearly ties between the two charts. This differs from \citet{stolberglarsen2021atlas}, who cross charts directly at individual data points whose routing exceeds a threshold in both charts, so that the correspondence between charts exists only at those points. A portal instead is a region of the overlap, and the transition map of \autoref{eq:transition} lets charts be crossed from anywhere within it.

\begin{figure}[t]
    \centering
    \includegraphics[width=\linewidth]{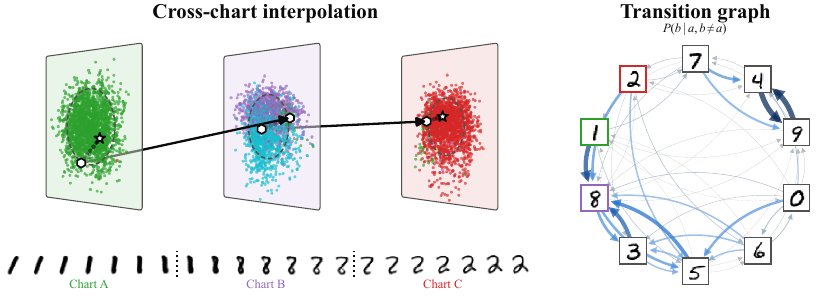}
    \caption{
        An atlas of chart autoencoders can create smooth latent interpolation within each chart as well as across multiple charts.
        After training, we precompute a transition graph (right).
        Interpolation between any two examples in any two charts can then be achieved by following the shortest path in that graph.
        Within each chart's latent space, we interpolate linearly (left).
    }
    \label{fig:interpolation}
\end{figure}

\paragraph{Transition measure and chart graph.}
To interpolate between points in different charts, we need to know which charts are adjacent. We use the model's own generative process to determine this. A sample drawn from chart $a$ is routed by $q_\phi$ to some chart $b$ with some probability which defines the Markov kernel
\begin{equation}
P(b\mid a) = \int q_\phi\bigl(b \mid g_{a}(\mathbf{z})\bigr) \ \tilde{p}(\mathbf{z}\mid a) \ d\mathbf{z}.
\label{eq:kernel}
\end{equation}
We estimate this probability with Monte Carlo. The kernel is row-stochastic, and its diagonal is the probability of staying in chart $a$. Moreover, it is generally asymmetric, since $P(b\mid a)$ and $P(a\mid b)$ involve different decoders. The charts and the off-diagonal entries of $P$ form a directed graph, the chart graph, with an edge $a\to b$ of weight $P(b\mid a)$ (\autoref{fig:interpolation}, right). Its integrand is the same routing probability whose tie defines the portal $q_{ab}$ in \autoref{eq:portal-objective}, so every edge of the graph comes with a portal through which it can be traversed.

\paragraph{Interpolation.}
With the transition kernel (\autoref{eq:kernel}) precomputed, interpolating between points in charts $a$ and $b$ reduces to a shortest-path problem on the chart graph with edge costs $-\log P(b\mid a)$. Since costs are additive, a chart sequence $a = c_0 \to c_1 \to \cdots \to c_K = b$ has cost $-\sum_k \log P(c_{k+1}\mid c_k)$, which corresponds to the negative log-probability of that sequence under the Markov chain defined by $P$. The shortest path, found with Dijkstra's algorithm, is therefore the most probable route from $a$ to $b$. It passes through zero or more intermediate charts, entering and leaving each through a portal. Within a chart we interpolate linearly to the portal of the next edge, represented by its mean $\mathbf{m}_{c_k c_{k+1}}$ or a sample drawn from $q_{c_k c_{k+1}}$, and transfer through the transition map of \autoref{eq:transition}.

\section{Results}
\label{sec:results}

We train our model on three image datasets and compare it with a standard Variational Auto-Encoder (VAE) and VaDE~\citep{jiang2016variational}.
We also attempted to compare it with CAE~\citep{schonsheck2019chart} and
MLMM-VAE~\citep{alberti2024manifold}, but did not succeed in fitting these models to the datasets: one or two charts covered the entire dataset, while the remaining charts were inactive.

All methods are scaled to use a similar number of parameters (model size), and hyperparameters are selected for each method and each dataset via grid search.
See \autoref{app:experimental-setup} for a detailed description of the experimental setup.
Additional results are available in \autoref{app:additional-results}, including ablations, reconstructions and samples.
Throughout this paper, 2D visualizations of higher-dimensional latent spaces use the first two PCA directions.

\paragraph{Geometric examples.}
In a geometric setting, such as the double torus shown in \autoref{fig:teaser}, our charts cover the surface with low distortion and clear separation.
Interpolation trajectories follow the data manifold closely, as shown in \autoref{fig:double-torus-interpolation-trajectories}.

\paragraph{Reconstruction \& generation.}
Our method faithfully reconstructs input examples, as shown in \autoref{fig:reconstruction-comparison}.
Reconstruction performance for our method and the capacity-matched baselines is also reported in \autoref{tab:vade_paper_table}.
For reconstruction alone, the vanilla VAE tends to perform best, which can be attributed to its larger decoder capacity in comparison to a single chart.
This tendency is reversed in image generation, where our method outperforms the two baselines.
This can be explained by the better latent structure and the lower incidence of samples drawn from low-density regions in the latent space (see \autoref{fig:results-chart-specialization}).
We can see that each chart's latent space is closer to a Gaussian than the combined latent space of the VAE.
As a result, the generated samples have fewer artifacts.
More examples generated with our model can be found in  \autoref{fig:teaser} and \autoref{fig:samples} in the appendix.

\begin{figure}[tb]
    \centering
    \includegraphics[width=\linewidth]{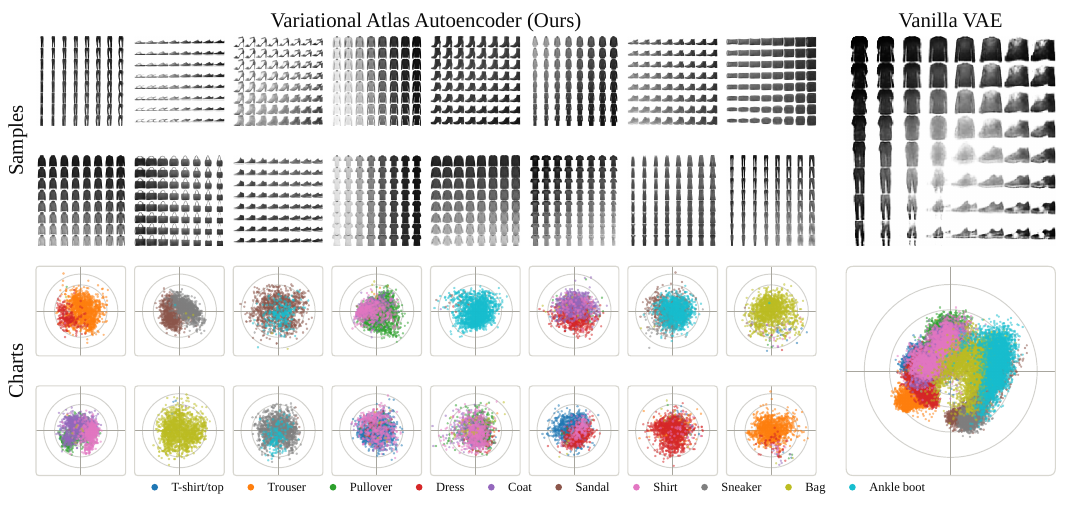}
    \caption{
        Sampling a model trained on Fashion-MNIST~\citep{fashionMNIST} along the two dominant latent PCA axes generates fewer degenerate samples than a parameter-matched vanilla VAE~\citep{kingma2013auto}.
        The charts also specialize without supervision, e.g. by placing bags without handles in a separate chart.
        For readability, we show 16 out of the 20 charts.
    }
    \label{fig:results-chart-specialization}
\end{figure}

\begin{figure}[tb]
    \centering
    \begin{subfigure}[b]{\textwidth}
        \centering
        \includegraphics[width=\linewidth]{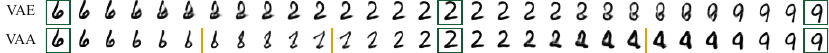}
        \subcaption{MNIST (10 charts)}
        \label{fig:mnist-interpolation}
    \end{subfigure}
    \begin{subfigure}[b]{\textwidth}
        \centering
        \includegraphics[width=\linewidth]{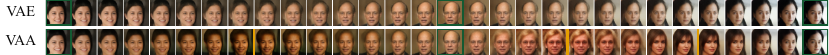}
        \subcaption{CelebA (10 charts)}
        \label{fig:celeba-interpolation}
    \end{subfigure}
     \begin{subfigure}[b]{\textwidth}
        \centering
        \includegraphics[width=\linewidth]{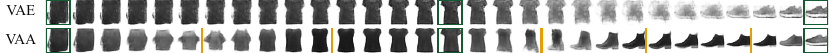}
        \subcaption{Fashion-MNIST (20 charts)}
        \label{fig:fashion-interpolation}
    \end{subfigure}
    \caption{
        Interpolations between sequences of randomly selected training samples (green).
        Our method may walk through one or more charts via portals (\autoref{sec:portals}), denoted by orange vertical bars.
        The vanilla VAE traverses its single global latent space in a straight line, which leads to several nonsensical generations along the way.
        In contrast, our method consistently stays on the data manifold, even if it entails covering more distance in latent space.
        \label{fig:interpolation-comparison}
    }
\end{figure}

\begin{table}[tbp]
\centering
\small
\caption{
    Sampling and reconstruction quality for a plain VAE~\citep{kingma2013auto}, VaDE~\citep{jiang2016variational}, and our atlas model (10 charts), all using the same number of parameters and training iterations.
    Quality is measured by the Fr\'echet Inception Distance (Inception-v3, \citet{szegedy2016rethinking}), averaged over three runs (mean $\pm$ std).
    Lower is better.
}
\label{tab:vade_paper_table}
\begin{tabular}{lcccccc}
\toprule
& \multicolumn{3}{c}{Sample FID} & \multicolumn{3}{c}{Reconstruction FID} \\
\cmidrule(lr){2-4} \cmidrule(lr){5-7}
Dataset & Plain VAE & VaDE & VAA (ours) & Plain VAE & VaDE & VAA (ours) \\
\midrule
MNIST & 19.84 $\pm$ 0.34 & \underline{14.95 $\pm$ 0.54} & \textbf{13.74 $\pm$ 0.55} & \textbf{16.81 $\pm$ 0.23} & 20.54 $\pm$ 0.12 & \underline{18.25 $\pm$ 0.13} \\
Fashion & 83.27 $\pm$ 1.62 & \underline{45.07 $\pm$ 1.62} & \textbf{37.73 $\pm$ 1.12} & \textbf{31.27 $\pm$ 0.32} & 43.49 $\pm$ 0.15 & \underline{32.83 $\pm$ 0.38} \\
CelebA & 30.43 $\pm$ 0.59 & \underline{28.01 $\pm$ 0.26} & \textbf{27.84 $\pm$ 0.42} & \textbf{16.66 $\pm$ 0.18} & 22.09 $\pm$ 0.07 & \underline{17.24 $\pm$ 0.25} \\
\bottomrule
\end{tabular}
\end{table}

\paragraph{Self-supervised structure discovery.}
Our variational chart selection objective uncovers dataset structure by regularizing reconstruction smoothness while allowing disconnected charts when needed.
Without supervision, our 20-chart VAA uncovers subclasses in Fashion-MNIST and other datasets (e.g.,\ sandals and shoes with heels, see \autoref{fig:results-chart-specialization}).
Moreover, we observe that the data closely follow the Gaussian prior in the latent domain, whereas this assumption breaks down in the capacity-matched VAE.

\paragraph{Interpolation.}
Clustering is not the only structure our method discovers. The transition kernel (\autoref{eq:kernel}) connects the charts into a graph that reflects semantic similarity without any supervision.
On MNIST (\autoref{fig:interpolation}), charts of similar digits such as 4 and 9 are strongly connected, while dissimilar ones such as 2 and 0 are nearly disconnected.
Interpolating along this graph (see \autoref{fig:interpolation}, \autoref{fig:interpolation-comparison}) yields realistic samples along the entire path, including the passage through intermediate charts.
In contrast, the capacity-matched VAE produces degenerate samples whenever its straight-line path crosses a region of low data density.
The \methodnamesingular\ avoids such regions because every segment of its path stays within a chart's latent support and hence close to the data manifold.

\section{Conclusion}
We have introduced an end-to-end probabilistic method to train variational atlas autoencoders, which are inspired by the manifold hypothesis.
We proposed latent-evidence-based routing, which connects the probability of each data point under each chart's prior to the confidence that the chart should encode it.
Finally, we showed how the mixture of charts can be turned into an atlas with a post-training step, allowing cross-chart interpolation trajectories that remain close to the data manifold.
Our method self-supervises the input data, resulting in smooth charts while staying faithful to the data manifold, and enables high-quality reconstruction and generation. Due to our chart selection not all model parameters are active during sampling, resulting in a lower sampling cost while outperforming relevant baselines on sampling quality.

In future work, determining an adequate chart count automatically, for example by adding or removing charts adaptively during training based on the data manifold's complexity, would improve the method's ease of use and utility.
Training can be further accelerated by training a lightweight routing proxy model, so that only the encoders in the most likely subset need to be evaluated, rather than all of them.
Next, because a single transition point is selected for each pair of charts, our cross-chart interpolations sometimes cover more distance than needed.
Introducing more portals should help minimize total interpolation distance.
Finally, scaling the method to larger datasets will likely require combining it with more advanced architectures such as latent-space diffusion or vector-quantized VAEs.

\subsection*{AI use statement}
In this work, we used large language models (LLMs) to prototype and develop code and to verify the concepts presented in this paper.
We have not used these tools to generate the core idea or the mathematics associated with it.
Additionally, we used LLM tools to help find and summarize related work, revise portions of the manuscript, and write code to build figures.
We have reviewed all AI-assisted work for correctness.
In particular, LLM-generated code was verified and tested by one author.
We take responsibility for the final content of this work, including text, claims, or artifacts produced with the aid of generative AI.

\subsection*{Ethics statement}
While we do not foresee risks introduced by this method beyond those shared by all generative methods, foundational research can have a wide range of uses in subsequent work that translates it into practical applications.
We employed standard academic datasets and performed all experiments \emph{in silico}.

\bibliography{iclr2027_conference}

@article{zhang2017deep,
  title={Deep unsupervised clustering using mixture of autoencoders},
  author={Zhang, Dejiao and Sun, Yifan and Eriksson, Brian and Balzano, Laura},
  journal={arXiv preprint arXiv:1712.07788},
  year={2017}
}

@inproceedings{ArvanitidisHH18,
  author       = {Georgios Arvanitidis and
                  Lars Kai Hansen and
                  S{\o}ren Hauberg},
  title        = {Latent Space Oddity: on the Curvature of Deep Generative Models},
  booktitle    = {6th International Conference on Learning Representations, {ICLR} 2018,
                  Vancouver, BC, Canada, April 30 - May 3, 2018, Conference Track Proceedings},
  publisher    = {OpenReview.net},
  year         = {2018},
  bibsource    = {dblp computer science bibliography, https://dblp.org}
}

@InProceedings{interp_oring21a,
  title = 	 {Autoencoder Image Interpolation by Shaping the Latent Space},
  author =       {Oring, Alon and Yakhini, Zohar and Hel-Or, Yacov},
  booktitle = 	 {Proceedings of the 38th International Conference on Machine Learning},
  pages = 	 {8281--8290},
  year = 	 {2021},
  volume = 	 {139},
  series = 	 {Proceedings of Machine Learning Research},
  month = 	 {18--24 Jul},
  publisher =    {PMLR},
  }

@article{vahdat2021score,
  title={Score-based generative modeling in latent space},
  author={Vahdat, Arash and Kreis, Karsten and Kautz, Jan},
  journal={Advances in Neural Information Processing Systems},
  volume={34},
  pages={11287--11302},
  year={2021}
}

@inproceedings{zhang2018unreasonable,
  title={The unreasonable effectiveness of deep features as a perceptual metric},
  author={Zhang, Richard and Isola, Phillip and Efros, Alexei A and Shechtman, Eli and Wang, Oliver},
  booktitle={2018 IEEE/CVF conference on computer vision and pattern recognition},
  pages={586--595},
  year={2018},
  organization={IEEE}
}

@article{stolberglarsen2021atlas,
title = {Atlas generative models and geodesic interpolation},
journal = {Image and Vision Computing},
volume = {122},
pages = {104433},
year = {2022},
author = {Jakob Stolberg-Larsen and Stefan Sommer},
}

@inproceedings{kingma2013auto,
  author       = {Diederik P. Kingma and
                  Max Welling},
  title        = {Auto-Encoding Variational {Bayes}},
  booktitle    = {2nd International Conference on Learning Representations, {ICLR}},
  year         = {2014},
  bibsource    = {dblp computer science bibliography, https://dblp.org}
}

@inproceedings{radford16dcgan,
  author       = {Alec Radford and
                  Luke Metz and
                  Soumith Chintala},
  title        = {Unsupervised Representation Learning with Deep Convolutional Generative
                  Adversarial Networks},
  booktitle    = {4th International Conference on Learning Representations, {ICLR}},
  year         = {2016},
  bibsource    = {dblp computer science bibliography, https://dblp.org}
}

@inproceedings{szegedy2016rethinking,
  title={Rethinking the inception architecture for computer vision},
  author={Szegedy, Christian and Vanhoucke, Vincent and Ioffe, Sergey and Shlens, Jon and Wojna, Zbigniew},
  booktitle={Proceedings of the IEEE conference on computer vision and pattern recognition},
  pages={2818--2826},
  year={2016}
}

@article{lecun1998gradient,
  title={Gradient-based learning applied to document recognition},
  author={LeCun, Yann and Bottou, L{\'e}on and Bengio, Yoshua and Haffner, Patrick},
  journal={Proceedings of the IEEE},
  volume={86},
  number={11},
  pages={2278--2324},
  year={1998},
  publisher={IEEE}
}

@inproceedings{liu2015faceattributes,
  title = {Deep Learning Face Attributes in the Wild},
  author = {Liu, Ziwei and Luo, Ping and Wang, Xiaogang and Tang, Xiaoou},
  booktitle = {Proceedings of International Conference on Computer Vision (ICCV)},
  year = {2015}
}

@article{fashionMNIST,
  author       = {Han Xiao and Kashif Rasul and Roland Vollgraf},
  title        = {{Fashion-MNIST}: a Novel Image Dataset for Benchmarking Machine Learning Algorithms},
  date         = {2017-08-28},
  year         = {2017},
  eprint       = {cs.LG/1708.07747},
  journal={arXiv preprint arXiv:1708.07747},

}

@article{schonsheck2019chart,
  title={Chart auto-encoders for manifold structured data},
  author={Schonsheck, Stefan and Chen, Jie and Lai, Rongjie},
  journal={arXiv preprint arXiv:1912.10094},
  year={2019}
}

@inproceedings{tanielian2020learning,
  title={Learning disconnected manifolds: a no {GAN}'s land},
  author={Tanielian, Ugo and Issenhuth, Thibaut and Dohmatob, Elvis and Mary, J{\'{e}}r{\'{e}}mie},
  booktitle={International Conference on Machine Learning},
  pages={9418--9427},
  year={2020},
  organization={PMLR}
}

@inproceedings{brown2022verifying,
  author       = {Bradley C. A. Brown and
                  Anthony L. Caterini and
                  Brendan Leigh Ross and
                  Jesse C. Cresswell and
                  Gabriel Loaiza{-}Ganem},
  title        = {Verifying the Union of Manifolds Hypothesis for Image Data},
  booktitle    = {The Eleventh International Conference on Learning Representations,
                  {ICLR}},
  year         = {2023},
}

@article{ye2021deep,
  title={Deep mixture generative autoencoders},
  author={Ye, Fei and Bors, Adrian G.},
  journal={IEEE Transactions on Neural Networks and Learning Systems},
  volume={33},
  number={10},
  pages={5789--5803},
  year={2022},
  publisher={IEEE}
}

@inproceedings{jiang2016variational,
author = {Jiang, Zhuxi and Zheng, Yin and Tan, Huachun and Tang, Bangsheng and Zhou, Hanning},
title = {Variational Deep Embedding: An Unsupervised and Generative Approach to Clustering},
year = {2017},
isbn = {9780999241103},
publisher = {AAAI Press},
booktitle = {Proceedings of the 26th International Joint Conference on Artificial Intelligence},
pages = {1965–1972},
numpages = {8},
location = {Melbourne, Australia},
series = {IJCAI'17},
}

@book{lee2012smooth,
    author    = {Lee, John M.},
    title     = {Introduction to Smooth Manifolds},
    edition   = {2},
    series    = {Graduate Texts in Mathematics},
    volume    = {218},
    publisher = {Springer},
    year      = {2012}
}

@article{fefferman2016testing,
    author  = {Fefferman, Charles and Mitter, Sanjoy and Narayanan, Hariharan},
    title   = {Testing the manifold hypothesis},
    journal = {Journal of the American Mathematical Society},
    volume  = {29},
    number  = {4},
    pages   = {983--1049},
    year    = {2016}
}

@article{alberti2024manifold,
  title={Manifold learning by mixture models of {VAE}s for inverse problems},
  author={Alberti, Giovanni S. and Hertrich, Johannes and Santacesaria, Matteo and Sciutto, Silvia},
  journal={Journal of Machine Learning Research},
  volume={25},
  number={202},
  pages={1--35},
  year={2024}
}

@article{batson2021topological,
  title={Topological obstructions to autoencoding},
  author={Batson, Joshua and Haaf, C. Grace and Kahn, Yonatan and Roberts, Daniel A.},
  journal={Journal of High Energy Physics},
  volume={2021},
  number={4},
  pages={280},
  year={2021},
  publisher={Springer}
}

@inproceedings{hoffman2016elbo,
    author    = {Hoffman, Matthew D. and Johnson, Matthew J.},
    title     = {{ELBO} surgery: yet another way to carve up the variational evidence lower bound},
    booktitle = {Workshop on Advances in Approximate Bayesian Inference, NeurIPS},
    year      = {2016}
}

@inproceedings{rombach2022high,
  title={High-resolution image synthesis with latent diffusion models},
  author={Rombach, Robin and Blattmann, Andreas and Lorenz, Dominik and Esser, Patrick and Ommer, Bj{\"o}rn},
  booktitle={2022 IEEE/CVF conference on computer vision and pattern recognition (CVPR)},
  pages={10674--10685},
  year={2022}
}

@article{brehmer2020flows,
  title={Flows for simultaneous manifold learning and density estimation},
  author={Brehmer, Johann and Cranmer, Kyle},
  journal={Advances in neural information processing systems},
  volume={33},
  pages={442--453},
  year={2020}
}

@inproceedings{teh2003automatic,
  author       = {Yee Whye Teh and
                  Sam T. Roweis},
  title        = {Automatic Alignment of Local Representations},
  booktitle    = {Advances in Neural Information Processing Systems},
  pages        = {841--848},
  publisher    = {{MIT} Press},
  year         = {2002},
}

@inproceedings{brand2003charting,
  author    = {Matthew Brand},
  title     = {Charting a Manifold},
  booktitle = {Advances in Neural Information Processing Systems},
  volume    = {15},
  pages     = {961--968},
  year      = {2002}
}

@article{tipping1999mixtures,
  author  = {Michael E. Tipping and Christopher M. Bishop},
  title   = {Mixtures of Probabilistic Principal Component Analyzers},
  journal = {Neural Computation},
  volume  = {11},
  number  = {2},
  pages   = {443--482},
  year    = {1999}
}

@techreport{ghahramani1996em,
  author      = {Zoubin Ghahramani and Geoffrey E. Hinton},
  title       = {The {EM} Algorithm for Mixtures of Factor Analyzers},
  institution = {Department of Computer Science, University of Toronto},
  number      = {CRG-TR-96-1},
  year        = {1996}
}

@inproceedings{arvanitidis2020geometrically,
  author       = {Georgios Arvanitidis and
                  S{\o}ren Hauberg and
                  Bernhard Sch{\"{o}}lkopf},
  title        = {Geometrically Enriched Latent Spaces},
  booktitle    = {The 24th International Conference on Artificial Intelligence and Statistics},
  series       = {Proceedings of Machine Learning Research},
  volume       = {130},
  pages        = {631--639},
  publisher    = {{PMLR}},
  year         = {2021},
  bibsource    = {dblp computer science bibliography, https://dblp.org}
}
\bibliographystyle{iclr2027_conference}

\newpage
\appendix
{\huge{Appendices}}
\section{Preliminaries}
\label{app:preliminaries}

We collect the geometric notions used in the main text and fix notation. Everything in \autoref{app:manifolds} is standard; we follow the conventions of \citet{lee2012smooth}.

\subsection{Manifolds, charts, and atlases}
\label{app:manifolds}

\paragraph{Manifolds.}
A \emph{$k$-dimensional manifold} is a space $\mathcal{M}$ that is locally Euclidean: every point has an open neighborhood homeomorphic to an open subset of $\mathbb{R}^k$.
Throughout, we assume that the \emph{manifold hypothesis} \citep{fefferman2016testing} holds: the data lie on (or close to) a manifold $\mathcal{M} \subset \mathbb{R}^d$ with $k \ll d$.

\paragraph{Charts.}
A \emph{chart} is a pair $(U, \varphi)$ where $U \subset \mathcal{M}$ is open and $\varphi\colon U \to \mathbb{R}^k$ is a homeomorphism onto an open subset of $\mathbb{R}^k$. The map $\varphi$ assigns local coordinates to points of $U$, and its inverse $\varphi^{-1}$ parametrizes $U$. In autoencoder terms, the encoder plays the role of $\varphi$ and the decoder that of $\varphi^{-1}$.

\paragraph{Atlases and transition maps.}
An \emph{atlas} is a collection of charts $\{(U_c, \varphi_c)\}_{c=1}^{C}$ whose domains cover $\mathcal{M}$. Where two charts overlap, their coordinates are related by the \emph{transition map}
\begin{equation}
    \psi_{ab} = \varphi_b \circ \varphi_a^{-1},
    \label{eq:prelim:transition}
\end{equation}
defined on $\varphi_a(U_a \cap U_b)$. Transition maps are themselves homeomorphisms, satisfy $\psi_{ba} = \psi_{ab}^{-1}$, and compose consistently on triple overlaps: $\psi_{bc} \circ \psi_{ab} = \psi_{ac}$. \autoref{sec:portals} measures how well the learned charts satisfy these conditions.

\paragraph{Ball charts.}
The charts learned in this paper are slightly more restrictive than the general definition. Each decoder is defined on all of $\mathbb{R}^L$ and the latent prior $\mathcal{N}(\mathbf{0}, \mathbf{I})$ has full support, so each chart domain is (approximately) a deformed copy of $\mathbb{R}^L$ rather than an arbitrary open set. We refer to such charts as \emph{ball charts}.

\subsection{Why a single chart is not enough}
\label{app:onechart}

A standard VAE is a single ball chart: one decoder maps a connected latent space $\mathbb{R}^L$ onto the data.
This can only faithfully represent $\mathcal{M}$ if $\mathcal{M}$ itself looks like $\mathbb{R}^L$, which excludes most manifolds of interest.

\begin{proposition}
\label{prop:onechart}
Let $g\colon \mathbb{R}^L \to \mathbb{R}^d$ be continuous and let $\mathcal{M} \subset \mathbb{R}^d$ be compact or disconnected. Then $g$ is not a homeomorphism onto $\mathcal{M}$. In particular, $\mathcal{M}$ cannot be covered by a single ball chart.
\end{proposition}

\begin{proof}
Homeomorphisms preserve compactness and connectedness \citep{lee2012smooth}. Since $\mathbb{R}^L$ is connected and not compact, so is any homeomorphic image of it.
\end{proof}

Both hypotheses are typical of real data. Bounded data such as images lie on compact sets, and multiclass data such as MNIST lie on several disconnected pieces. In either case, at least two charts are needed. The sphere, for example, requires exactly two. The proposition also explains the spurious interpolants of single-latent models discussed in \autoref{sec:intro}. The image $g(\mathbb{R}^L)$ of a single decoder is connected, so if it covers a disconnected $\mathcal{M}$ it must also pass through points outside $\mathcal{M}$. Off-manifold decoding along latent paths between components is therefore unavoidable, not a failure of optimization.

\subsection{Autoencoders as approximate charts}
\label{app:autoencoders}

\paragraph{Autoencoders.}
An autoencoder consists of an encoder $f\colon \mathbb{R}^d \to \mathbb{R}^L$ and a decoder $g\colon \mathbb{R}^L \to \mathbb{R}^d$ trained to minimize the reconstruction error $\mathbb{E}_{p(\mathbf{x})}\|\mathbf{x} - g(f(\mathbf{x}))\|^2$. With perfect reconstruction, $(f, g)$ would be a chart with $U = \mathcal{M}$. Reconstruction is never perfect, so we call $(f, g)$ an \emph{approximate chart}: a chart up to the reconstruction error. The atlas properties studied in \autoref{sec:portals} hold to the same accuracy.

\paragraph{Variational autoencoders.}
A VAE \citep{kingma2013auto} places a prior $p(\mathbf{z}) = \mathcal{N}(\mathbf{0}, \mathbf{I})$ on the latent, models $p_\theta(\mathbf{x} \mid \mathbf{z})$ with the decoder, and approximates the posterior by $q_\phi(\mathbf{z} \mid \mathbf{x})$ using the encoder. It is trained by maximizing the evidence lower bound
\begin{equation}
    \log p_\theta(\mathbf{x}) \ge \mathbb{E}_{q_\phi(\mathbf{z} \mid \mathbf{x})}[\log p_\theta(\mathbf{x} \mid \mathbf{z})] - D_{\mathrm{KL}}\big(q_\phi(\mathbf{z} \mid \mathbf{x}) \,\|\, p(\mathbf{z})\big),
    \label{eq:prelim:elbo}
\end{equation}
which is the single-chart case $C = 1$ of \autoref{eq:multichart-elbo}.

\paragraph{Chart autoencoders.}
Chart autoencoders \citep{schonsheck2019chart} use $C$ encoder--decoder pairs, one per chart. Since a plain autoencoder has no likelihood, they require a separate network and loss to decide which chart a point belongs to. In our model this decision follows from the ELBO itself (\autoref{sec:methodology}).

\subsection{Notation}
\label{app:notation}

For a map $f$ and a distribution $\mu$, $f_\#\mu$ denotes the pushforward of $\mu$ under $f$, i.e. the distribution of $f(\mathbf{u})$ for $\mathbf{u} \sim \mu$. For a distribution $\nu$ on $\mathbb{R}^L$, $\mathrm{mm}(\nu)$ denotes the Gaussian with the same mean and covariance. A set of non-negative functions $\{w_c\}$ with $\sum_c w_c(\mathbf{x}) = 1$ for every $\mathbf{x}$ is a \emph{partition of unity}; the routing posterior $q_\phi(c \mid \mathbf{x})$ is one by construction.

\section{Detailed derivation of the objective function}
\label{app:full-derivation}

We provide the detailed derivation of the objective function.
To model $p_\theta(\x)$ using multiple charts, we introduce a discrete chart variable $c\in\{1,\ldots,C\}$. As a modeling assumption, each chart is assigned the same prior probability, so that $p(c)=1/C$. Each chart is further equipped with a continuous $L$-dimensional latent variable $\z_c\in\mathbb{R}^L$ with conditional prior $p(\z_c\mid c)=\mathcal{N}(\mathbf{0},\mathbf{I})$.
Repeating \autoref{eq:data-distr-charts} for convenience, the resulting modeled data distribution is then given by:

\begin{equation} \label{eq:data-distr-charts-appendix}
    p_\theta(\x)
    =
    \sum_{c=1}^{C} p(c)
    \int
    p_\theta(\x\mid\z_c,c)\,
    p(\z_c\mid c)\,
    \mathrm{d}\z_c .
\end{equation}

Given an observation $\x$, inference requires the posterior distribution over the chart assignment $p_\theta(c\mid\x)$ and the chart-specific latent posterior $p_\theta(\z_c\mid c,\x)$. Since these posterior distributions are generally intractable, we introduce the variational approximations $q_\phi(c\mid\x)$ and $q_\phi(\z_c\mid c,\x)$.
Together, they define the approximate posterior (repeating \autoref{eq:approximate-posterior} for convenience)

\begin{equation} \label{eq:approximate-posterior-appendix}
    q_\phi(c,\z_c\mid\x)
    =
    q_\phi(c\mid\x)\,
    q_\phi(\z_c\mid c,\x),
\end{equation}

which approximates the true posterior $p_\theta(c,\z_c\mid\x)$.

Direct evaluation and optimization of the marginal log-likelihood in \autoref{eq:data-distr-charts-appendix} are generally intractable because the integral over the continuous latent variables cannot usually be computed in closed form. As in the single-chart VAE setting~\citep{kingma2013auto}, we therefore optimize a variational lower bound on \(\log p_\theta(\x)\).

Using $q_\phi(c,\z_c\mid\x)=q_\phi(c\mid\x)q_\phi(\z_c\mid c,\x)$, the marginal likelihood can be rewritten as
\begin{align}
    \log p_\theta(\x)
    &=
    \log
    \sum_{c=1}^{C}
    \int
    q_\phi(c,\z_c\mid\x)
    \frac{
        p_\theta(\x,\z_c,c)
    }{
        q_\phi(c,\z_c\mid\x)
    }
    \,\mathrm{d}\z_c .
\end{align}
Applying Jensen's inequality yields
\begin{align*}
    \log p_\theta(\x)
    &\geq
    \sum_{c=1}^{C}
    q_\phi(c\mid\x)
    \int
    q_\phi(\z_c\mid c,\x)
    \log
    \frac{
        p_\theta(\x,\z_c,c)
    }{
        q_\phi(c,\z_c\mid\x)
    }
    \,\mathrm{d}\z_c
    \\
    &=
    \mathbb{E}_{c\sim q_\phi(c\mid\x)}
    \left[
        \mathbb{E}_{\z_c\sim q_\phi(\z_c\mid c,\x)}
        \left[
            \log
            \frac{
                p_\theta(\x,\z_c,c)
            }{
                q_\phi(c,\z_c\mid\x)
            }
        \right]
    \right].
\end{align*}
This defines the evidence lower bound for the multi-chart model,
\begin{equation*}
    \log p_\theta(\x)
    \geq
    \mathcal{L}_{\mathrm{ELBO}}(\x)
    =
    \mathbb{E}_{q_\phi(c,\z_c\mid\x)}
    \left[
        \log
        \frac{
            p_\theta(\x,\z_c,c)
        }{
            q_\phi(c,\z_c\mid\x)
        }
    \right].
\end{equation*}
The lower bound can equivalently be related to the exact posterior through
\begin{equation*}
    \log p_\theta(\x)
    =
    \mathcal{L}_{\mathrm{ELBO}}(\x)
    +
    D_{\mathrm{KL}}\!\left(
        q_\phi(c,\z_c\mid\x)
        \,\Vert\,
        p_\theta(c,\z_c\mid\x)
    \right).
\end{equation*}
Since the KL divergence is non-negative, it follows that
$\log p_\theta(\x)\geq\mathcal{L}_{\mathrm{ELBO}}(\x)$.
Using the factorization
$p_\theta(\x,\z_c,c) = p_\theta(\x\mid\z_c,c)\, p(\z_c\mid c)\, p(c)$ and
$q_\phi(c,\z_c\mid\x) = q_\phi(\z_c\mid c,\x)\, q_\phi(c\mid\x)$,
the ELBO can be written as
\begin{align*}
    \mathcal{L}_{\mathrm{ELBO}}(\x)
    &=
    \mathbb{E}_{c\sim q_\phi(c\mid\x)}
    \mathbb{E}_{\z_c\sim q_\phi(\z_c\mid c,\x)}
    \left[
        \log
        \frac{
            p_\theta(\x\mid\z_c,c)\,
            p(\z_c\mid c)\,
            p(c)
        }{
            q_\phi(\z_c\mid c,\x)\,
            q_\phi(c\mid\x)
        }
    \right]
    \\
    &=
    \mathbb{E}_{c\sim q_\phi(c\mid\x)}
    \mathbb{E}_{\z_c\sim q_\phi(\z_c\mid c,\x)}
    \left[
        \log p_\theta(\x\mid\z_c,c)
    \right]
    \nonumber\\
    &\quad+
    \mathbb{E}_{c\sim q_\phi(c\mid\x)}
    \mathbb{E}_{\z_c\sim q_\phi(\z_c\mid c,\x)}
    \left[
        \log
        \frac{
            p(\z_c\mid c)
        }{
            q_\phi(\z_c\mid c,\x)
        }
    \right]
    \nonumber\\
    &\quad+
    \mathbb{E}_{c\sim q_\phi(c\mid\x)}
    \left[
        \log
        \frac{
            p(c)
        }{
            q_\phi(c\mid\x)
        }
    \right].
\end{align*}
Expressing the last two terms as KL divergences gives
\begin{align*}
    \mathcal{L}_{\mathrm{ELBO}}(\x)
    &=
    \mathbb{E}_{c\sim q_\phi(c\mid\x)}
    \mathbb{E}_{\z_c\sim q_\phi(\z_c\mid c,\x)}
    \left[
        \log p_\theta(\x\mid\z_c,c)
    \right]
    \nonumber\\
    &\quad-
    \mathbb{E}_{c\sim q_\phi(c\mid\x)}
    \left[
        D_{\mathrm{KL}}\!\left(
            q_\phi(\z_c\mid c,\x)
            \,\Vert\,
            p(\z_c\mid c)
        \right)
    \right]
    \nonumber\\
    &\quad-
    D_{\mathrm{KL}}\!\left(
        q_\phi(c\mid\x)
        \,\Vert\,
        p(c)
    \right).
\end{align*}
Writing the chart router as an explicit finite sum makes it clear that the per-chart ELBO
\begin{equation}
    \tag{\ref{eq:per-chart-elbo}}
    \mathcal{L}_c(\x) = \mathbb{E}_{q_\phi(\z_c \mid c, \x)}\!\left[\log p_\theta(\x \mid \z_c, c)\right] - D_{\mathrm{KL}}\!\left(q_\phi(\z_c \mid c, \x) \,\|\, p(\z_c \mid c)\right)
\end{equation}
is simply weighted by each chart's routing probability $q_\phi(c\mid\x)$, while the prior routing term $D_{\mathrm{KL}}(q_\phi(c\mid\x) \,\Vert\,p(c))$ regularizes the routing function, yielding \autoref{eq:multichart-elbo}.
In practice, we replace the log-likelihood by a reconstruction loss $\mathcal{L}_{\mathrm{recon}}(\x, g_c(\z_c))$, which defaults to $-\log p_\theta(\x\mid\z_c,c)$ but may also be a perceptual loss (\autoref{app:experimental-setup}), and weight the regularizers by the hyperparameters $\beta$ (KL weight).
This yields the per-chart loss
\begin{equation}
\label{eq:per-chart-loss}
\ell_c(\x)
=
\mathbb{E}_{q_\phi(\z_c \mid c,\x)}
\left[
    \mathcal{L}_{\mathrm{recon}}\!\left(\x, g_c(\z_c)\right)
\right]
+
\beta\, D_{\mathrm{KL}}\!\left(
    q_\phi(\z_c\mid c,\x)
    \,\Vert\,
    p(\z_c\mid c)
\right)
\end{equation}
and the training loss
\begin{equation}
\label{eq:weighted-loss}
\ell(\x)
=
 D_{\mathrm{KL}}\!\left(
    q_\phi(c\mid\x)
    \,\Vert\,
    p(c)
\right)
+ \sum_{c=1}^{C}
    q_\phi(c\mid\x)\, \ell_c(\x),
\end{equation}
where $q_\phi(c\mid\x)$ is the router of \autoref{eq:prior-marginal-routing} at temperature $T$. The loss coincides with $-\mathcal{L}_{\mathrm{ELBO}}$ from \autoref{eq:multichart-elbo} when $\beta = 1$ and $\mathcal{L}_{\mathrm{recon}} = -\log p_\theta$.
\section{Reducing the roundtrip error}
\label{app:roundtrip_error}
In the portal objective \autoref{eq:portal-objective}, the second term preserves $q_{ab}$ as a distribution.
This locates the overlap but does not make $\psi_{ab}$ invertible on it.
To fix this, we propose an optional lightweight fine-tuning step in which the expected pointwise defect,
\begin{equation}
  \mathcal{L}_{\mathrm{atlas}} =
  \mathbb{E}_{(a,b)\sim\pi}\Bigl[
    \mathbb{E}_{\mathbf{z}\sim q_{ab}}\|R_{ab}(\mathbf{z}) - \mathbf{z}\|_a^2
    + \mathbb{E}_{\mathbf{w}\sim(\psi_{ab})_\# q_{ab}}\|R_{ba}(\mathbf{w}) - \mathbf{w}\|_b^2
  \Bigr],
  \label{eq:atlas-loss}
\end{equation}
is minimized in both directions, jointly with the training loss \autoref{eq:weighted-loss} as a retention term.
Distances are measured in the Mahalanobis metric $\|\cdot\|_c$ defined above, so that a defect of one is the size of a typical latent and defects are comparable across charts.
The outer expectation is estimated by drawing pairs from $\pi$ with replacement and the inner expectations by reparameterized Monte Carlo with gradients through $\psi_{ab}$.
The retention term ensures that the model retains its sampling and reconstruction performance.
\section{Experimental setup}
\label{app:experimental-setup}

This appendix provides additional details on the setup summarized in \autoref{sec:results},
specifically on the architecture
(Sec.~\ref{sec:architecture}), the chart-balance term and its weight
$\gamma$ (Sec.~\ref{sec:balance}), per-dataset hyperparameters
(Secs.~\ref{sec:image-datasets}-\ref{sec:torus-setup}), and how the vanilla VAE and
VaDE baselines are made comparable to ours (Sec.~\ref{sec:baselines}).

\subsection{Network Architecture}
\label{sec:architecture}

We use two different network architectures: an MLP-only architecture for the double
torus dataset, and a convolutional architecture for the image-based
datasets.

\paragraph{MLP (double torus).} The encoders and decoders are
\texttt{depth}-layer, \texttt{hidden}-wide dense networks with ReLU
nonlinearities, where each encoder/decoder pair is private to its respective chart.
This can be seen as $C$ separate pairs of dense MLPs, one per chart; although in
practice we implement this as a single batched matrix multiply over
all $C$ charts at once, rather than looping over $C$ separate layers,
which is significantly faster
on a GPU. The encoder maps the 3D ambient coordinates that lie on the
surface of the torus to a
$2\times$\texttt{latent\_dim} latent vector containing the mean and log-variance; the decoder mirrors this
back to 3D, and its raw output is used directly as the reconstruction,
with no output nonlinearity (Sec.~\ref{sec:torus-setup}).

\paragraph{DCGAN (MNIST, Fashion-MNIST, CelebA).} For image datasets, we use an architecture that can exploit their spatial structure more effectively. The encoder and decoder follow
the convolutional architecture of~\cite{radford16dcgan},
with the channel width doubling at every downsampling stage and, conversely,
halving at every upsampling stage.
In every tuned configuration reported in this
paper, this entire convolutional network (encoder and decoder) is
shared identically across all charts (Table~\ref{tab:arch}). Charts only
differ at the very end: the shared encoder's final feature map is flattened
and passed through one linear layer per chart, mapping it to that chart's
own $2\times$\texttt{latent\_dim} latent vector containing the mean and log-variance; the decoder mirrors
this with one linear layer per chart that expands $z$ back to that same
feature-map size, before the (again shared) network decodes it into an
image.
These per-chart linear layers are the only chart-specific
parameters in this backbone. The VAE and VaDE baselines
(Sec.~\ref{sec:baselines}) reuse this same convolutional network at $C=1$.

The architecture also supports making some of the deepest convolutional
stages chart-specific as well, rather than only the two linear layers. Reducing the size of the shared encoder/decoder model improves the generation speed, but comes with the trade-off that each chart needs to learn the mapping from ambient to latent space from scratch, effectively re-learning the feature extractor.

\begin{table}[H]
    \centering
    \caption{
        Resolved DCGAN dimensions per dataset: channel widths of the shared convolutional network and the size of the feature map passed to the per-chart linear layers.
        \label{tab:arch}
    }
    \begin{tabular}{lrrr}
    \toprule
    & MNIST & Fashion-MNIST & CelebA \\
    \midrule
    Convolutional stages & 2 & 2 & 4 \\
    Channel widths & 64, 128 & 64, 128 & 64, 128, 256, 512 \\
    Bottleneck channels & 13 & 6 & 51 \\
    Bottleneck resolution & $7\times7$ & $7\times7$ & $4\times4$ \\
    \bottomrule
    \end{tabular}
\end{table}

\subsection{The Balance Weight $\gamma$}
\label{sec:balance}

Although \autoref{eq:multichart-elbo} already contains a routing KL toward the uniform chart prior,
we keep that term at weight $1$ and do not tune it. The hyperparameter
$\gamma$ instead weights a separate \emph{balance term}, which compares the
uniform distribution $U$ over the $C$ charts with the routing averaged over
a minibatch, $\bar{r}_c = \frac{1}{B}\sum_{i=1}^{B} q_\phi(c \mid x_i)$.
The loss we minimize is therefore
\begin{equation}
\mathcal{L} = -\mathcal{L}_{\mathrm{ELBO}} + \gamma\, D_{\mathrm{KL}}(U\,\|\,\bar{r}),
\label{eq:balance}
\end{equation}
with $\gamma = 50$ for the entire run in every configuration.

The two terms push in different ways. The routing KL of \autoref{eq:multichart-elbo} acts on each
sample and pulls every $q_\phi(c\mid x)$ toward uniform, which also
discourages confident chart assignments. The balance term acts only on the
batch average: individual points may still be routed sharply to a single
chart, as long as all charts are used about equally overall. To see why, note
that the batch-averaged routing KL splits as
$\frac{1}{B}\sum_i D_{\mathrm{KL}}(r_i\,\|\,U) = \hat{I}(X;C) +
D_{\mathrm{KL}}(\bar{r}\,\|\,U)$, with $r_i = q_\phi(c\mid x_i)$ and
$\hat{I}(X;C) = \frac{1}{B}\sum_i D_{\mathrm{KL}}(r_i\,\|\,\bar{r})$ the
empirical mutual information between inputs and charts. The first part
penalizes confident routing, the second unequal chart usage. The $\gamma$ term amplifies only the
latter (in reversed form), so charts are kept alive without blurring the assignments. This is the same role played by the auxiliary
load-balancing loss in related work~\citep{zhang2017deep}, which prevents a router from
sending all inputs to a few expert charts.

\subsection{Image Datasets}
\label{sec:image-datasets}

Hyperparameters (\autoref{tab:image-hyperparameters}) are selected independently per dataset via grid search. The VAE and VaDE baselines reuse the architecture of Sec.~\ref{sec:architecture} and the reconstruction loss below, with only width rescaled to match parameter count (Sec.~\ref{sec:baselines}).

\begin{table}[H]
    \centering
    \caption{
        Tuned hyperparameters for MNIST, Fashion-MNIST, and CelebA.
        \label{tab:image-hyperparameters}
    }
    \begin{tabular}{lrrr}
    \toprule
    Hyperparameter & MNIST & Fashion-MNIST & CelebA \\
    \midrule
    Charts & 10 & 20 & 10 \\
    Latent dimension & 16 & 16 & 128 \\
    Reconstruction loss & Bernoulli (BCE) & Bernoulli (BCE) & \makecell[tr]{L1 + LPIPS, scale $\times 3$} \\
    $T$ (routing temperature) & Fixed at 1.0 & Fixed at 1.0 & Annealed $6.0 \to 2.0$  \\
    $\beta$ (KL weight) & 0.5 & 0.1 & 0.1 \\
    $\gamma$ (load-balancing weight) & 50, entire run & 50, entire run & 50, entire run \\
    Learning rate & $5\times10^{-4}$ & $2\times10^{-3}$ & $5\times10^{-4}$ \\
    LR schedule & Cos-annealed to 1\%  & Cos-annealed to 1\% & Cos-annealed to 1\% \\
    Batch size & 256 & 256 & 128 \\
    Epochs & 200 & 200 & 200 \\
    \bottomrule
    \end{tabular}
\end{table}

A per-pixel likelihood scores each pixel independently, which rewards blurry reconstructions.
For CelebA specifically, we replace it with a mix of L1 error in pixel space and LPIPS~\citep{zhang2018unreasonable}, a common perceptual loss used by latent-diffusion autoencoders.
Rather than comparing raw pixels, LPIPS compares distances in a pretrained latent space that tracks human judgments more closely.
We weight the two terms so that each contributes about as many nats as the part of the BCE loss it replaces, split evenly, then scale the sum by 3 (see \autoref{tab:image-hyperparameters}).

Two optimizations keep this loss practical. LPIPS is by far the most
expensive term to compute, so we evaluate it once per image rather than
once per chart: every chart still gets its own L1 term, but all charts
share the LPIPS score of whichever chart reconstructs best. This leaves
the ranking between charts unchanged while avoiding the slowdown
from scoring every chart separately. We also run the LPIPS network at
reduced precision, which roughly halves its cost with no measurable effect
on training.

\subsection{Double Torus}
\label{sec:torus-setup}

The double torus is a synthetic point cloud sampled uniformly by area from
a genus-2 surface embedded in $\mathbb{R}^3$ (the surface of \autoref{fig:teaser}). Targets are
3D coordinates rather than pixel intensities, so we use a squared-error
reconstruction term directly on the decoder's raw output, rather than
passing it through a sigmoid first as we do for pixel intensities. This term is small relative to the KL
term and induces posterior collapse if left uncorrected. We counteract this
by scaling up the reconstruction term ($200\times$), with no other change to the loss,
routing, or training procedure. Table~\ref{tab:torus} lists the full
configuration.

\newcolumntype{R}[1]{>{\raggedleft\arraybackslash}p{#1}}
\begin{table}[H]
    \centering
    \caption{
        Tuned hyperparameters for the double torus.
        \label{tab:torus}
    }
    \begin{tabular}{lR{7.5cm}}
    \toprule
    Hyperparameter & Value \\
    \midrule
    Backbone & MLP, hidden width 512, depth 5 \\
    Charts & 4 \\
    Latent dimension & 2 \\
    Reconstruction loss & Squared error, no output squashing, scale $\times 200$ \\
    $T$ (routing temperature) & annealed $1.0 \to 0.2$ \\
    $\beta$ (KL weight) & 1.5 \\
    $\gamma$ (load-balancing weight) & 50, for the entire run \\
    Learning rate & $1\times10^{-3}$, cosine-annealed to 1\% \\
    Batch size & 2048 \\
    Epochs & 100 \\
    Surface points & $1{,}000{,}000$ \\
    \bottomrule
    \end{tabular}
\end{table}

\subsection{Baselines}
\label{sec:baselines}

Both baselines are given our model's total trainable parameter count on each dataset: we perform a binary search over the baseline's width parameter (channel count for DCGAN, hidden width for the MLP) to match that target, rather than altering depth or other architectural choices, and train for the same number of epochs with the same reconstruction loss.
The vanilla VAE control is simply $C=1$ (Sec.~A.3), so matching parameters amounts to retuning the backbone width at $C=1$.

For VaDE, we port the original authors' code to modern PyTorch, verify it reproduces their reported numbers, and then swap its Gaussian-mixture prior onto our own single-chart DCGAN backbone (Sec.~\ref{sec:architecture}) in place of their original autoencoder.
So the architecture, parameter count, and training length are identical across all three methods, and only the modeling choice (charts vs.\ GMM prior vs.\ single Gaussian prior) differs.

As noted in \autoref{sec:results}, we also attempted CAE~\citep{schonsheck2019chart} and MLMM-VAE~\citep{alberti2024manifold} but did not succeed in fitting either to the image-based datasets: one or two charts covered the whole dataset, while the remaining charts were effectively inactive. As a result these methods were only using a small subset of the available model-parameter capacity resulting in significantly higher FID when sampling new data from the model.
For that reason, we exclude them from the main comparison.

\section{Algorithm}
\label{app:algorithm}

\begin{algorithm}[ht]
\caption{Training procedure for \methodname. All $C$ encoders are evaluated for each minibatch to obtain the routing weights and per-chart losses; encoders and decoders are updated jointly.}
\label{alg:training}
\begin{algorithmic}[1]
\Require data $\{\mathbf{x}\}$, encoder--decoder pairs $\{(f_c, g_c)\}_{c=1}^C$, temperature $T$, weights $\beta, \gamma$
\While{not converged}
    \State sample a minibatch $\mathbf{x}$
    \For{$c = 1, \dots, C$}
        \State $(\boldsymbol{\mu}_c, \boldsymbol{\Sigma}_c) \gets f_c(\mathbf{x})$ \Comment{encode with every chart}
        \State $s_c \gets \log \mathcal{N}(\boldsymbol{\mu}_c;\, \mathbf{0},\, \mathbf{I} + \boldsymbol{\Sigma}_c)$ \Comment{latent evidence, \autoref{eq:prior-marginal-routing}}
        \State $\mathbf{z}_c \gets \boldsymbol{\mu}_c + \boldsymbol{\Sigma}_c^{1/2}\boldsymbol{\epsilon}$, $\boldsymbol{\epsilon} \sim \mathcal{N}(\mathbf{0}, \mathbf{I})$ \Comment{reparameterized sample}
        \State $\ell_c \gets \mathcal{L}_{\mathrm{recon}}(\mathbf{x}, g_c(\mathbf{z}_c)) + \beta\, D_{\mathrm{KL}}(q_\phi(\mathbf{z}_c \mid c, \mathbf{x}) \,\|\, p(\mathbf{z}_c \mid c))$ \Comment{per-chart loss, \autoref{eq:per-chart-loss}}
    \EndFor
    \State $q_\phi(c \mid \mathbf{x}) \gets \operatorname{softmax}_c\!\left(s_c / T + \log p(c)\right)$ \Comment{routing}
    \State $\ell \gets \sum_c q_\phi(c \mid \mathbf{x})\, \ell_c + \gamma\, D_{\mathrm{KL}}(q_\phi(c \mid \mathbf{x}) \,\|\, p(c))$ \Comment{weighted loss, \autoref{eq:weighted-loss}}
    \State update $\phi, \theta$ by gradient descent on $\ell$
\EndWhile
\end{algorithmic}
\end{algorithm}

\begin{algorithm}[ht]
\caption{Sampling from \methodname. Only the selected decoder is evaluated, so the cost is independent of $C$. With uniform $p(c)$ and standard normal priors, $\hat{p}(c) = 1/C$ and $(\mathbf{m}_c, \mathbf{S}_c) = (\mathbf{0}, \mathbf{I})$; the moment-matched values provide an optional correction described in \autoref{sec:sampling}.}
\label{alg:sampling}
\begin{algorithmic}[1]
\Require trained decoders $\{g_c\}_{c=1}^C$, chart weights $\hat{p}(c)$, sampling distributions $\{\tilde{p}(\mathbf{z}_c \mid c) = \mathcal{N}(\mathbf{m}_c, \mathbf{S}_c)\}_{c=1}^C$
\State $c \sim \hat{p}(c)$ \Comment{pick a chart}
\State $\mathbf{z}_c \sim \tilde{p}(\mathbf{z}_c \mid c)$ \Comment{sample its latent}
\State $\hat{\mathbf{x}} \gets g_c(\mathbf{z}_c)$ \Comment{decode}
\State \Return $\hat{\mathbf{x}}$
\end{algorithmic}
\end{algorithm}

Algorithm \autoref{alg:training} and \autoref{alg:sampling} summarize the training and sampling procedures. During training, every encoder is evaluated on each minibatch. The encoder outputs determine both the routing weights via the latent-evidence score of \autoref{eq:prior-marginal-routing} and, through a reparameterized sample, the per-chart losses of \autoref{eq:per-chart-loss}. These are combined into the weighted objective of \autoref{eq:weighted-loss}, and all encoders and decoders are updated jointly. Sampling requires only a chart index, a latent sample from that chart's prior, and a single decoder pass.

\section{Additional results}
\label{app:additional-results}

We present additional results for 10-chart variants of the variational atlas autoencoder trained on MNIST, Fashion-MNIST, and CelebA. \autoref{fig:reconstruction-comparison} depicts training samples and their reconstructions from the fully converged model. \autoref{fig:samples} depicts samples drawn from each of the 10 charts within each VAA.

\begin{table}[tbp]
\caption{Sample FID (mean $\pm$ std over seeds) by routing function, DCGAN backbone, $C=10$ charts, and iteration counts matched within each dataset. Lower is better; best in bold, second best underlined.}
\label{tab:routing_ablation_table}
\centering
\small
\begin{tabular}{lccccc}
\toprule
& & \multicolumn{4}{c}{Latent-evidence routing} \\
\cmidrule(lr){3-6}
Dataset & ELBO & $\mu$-based & Resampling & Expected log-prior & Prior-marginal \\
\midrule
MNIST & \textbf{11.80 $\pm$ 0.10} & 12.99 $\pm$ 0.22 & 13.79 $\pm$ 1.00 & 12.77 $\pm$ 0.33 & \underline{12.62 $\pm$ 0.53} \\
Fashion-MNIST & 43.64 $\pm$ 0.34 & \underline{38.42 $\pm$ 0.49} & 39.77 $\pm$ 1.46 & \textbf{38.03 $\pm$ 0.64} & 38.53 $\pm$ 0.45 \\
CelebA & 47.14 $\pm$ 0.19 & \textbf{43.46 $\pm$ 0.53} & 44.05 $\pm$ 0.24 & 43.89 $\pm$ 0.55 & \underline{43.69 $\pm$ 0.24} \\
\bottomrule
\end{tabular}
\end{table}

\paragraph{Choice of routing probability.}
\autoref{tab:routing_ablation_table} shows the effect of the routing function on the sample FID averaged over three seeds. The different variants of latent-evidence routing produce the best outcome in most cases. For our experiments, we choose prior-marginal routing, but we find that both $\mu$-based and expected log-prior routing produce comparable results.

\paragraph{Reconstruction error.}
\autoref{fig:reconstruction-comparison} depicts additional reconstruction samples computed with our 10-chart VAA and compares them with those from the capacity-matched VAE.

\begin{figure}[t]
    \centering
    \begin{subfigure}[b]{\textwidth}
        \centering
        \includegraphics[width=\linewidth]{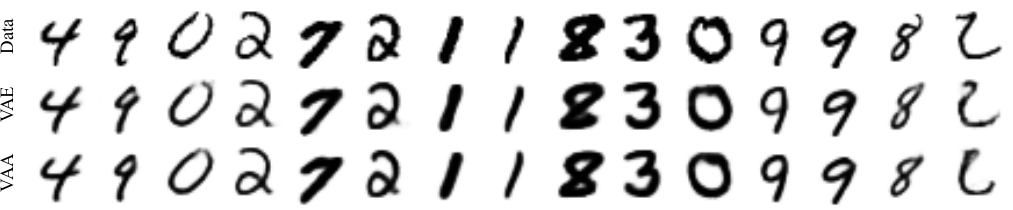}
        \subcaption{MNIST}
        \label{fig:mnist-recon}
    \end{subfigure}
    \begin{subfigure}[b]{\textwidth}
        \centering
        \includegraphics[width=\linewidth]{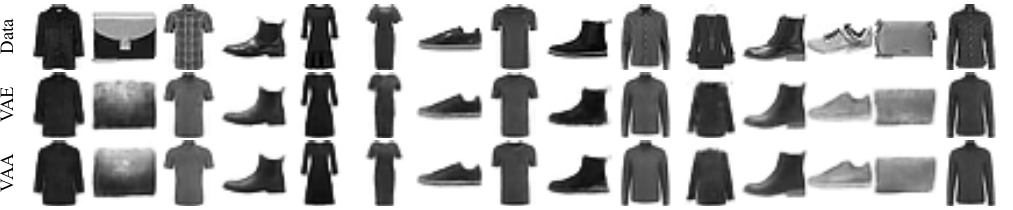}
        \subcaption{Fashion-MNIST}
        \label{fig:fashion-recon}
    \end{subfigure}
    \begin{subfigure}[b]{\textwidth}
        \centering
        \includegraphics[width=\linewidth]{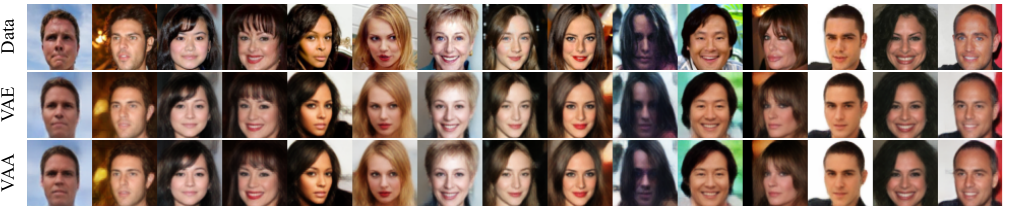}
        \subcaption{CelebA}
        \label{fig:celeba-recon}
    \end{subfigure}
    \caption{
        10-chart \methodname (ours) and 1-chart capacity-matched VAE reconstructions of randomly selected held-out inputs from MNIST, Fashion-MNIST and CelebA.
    }
    \label{fig:reconstruction-comparison}
\end{figure}

\begin{figure}[t]
    \centering
    \includegraphics[width=\linewidth]{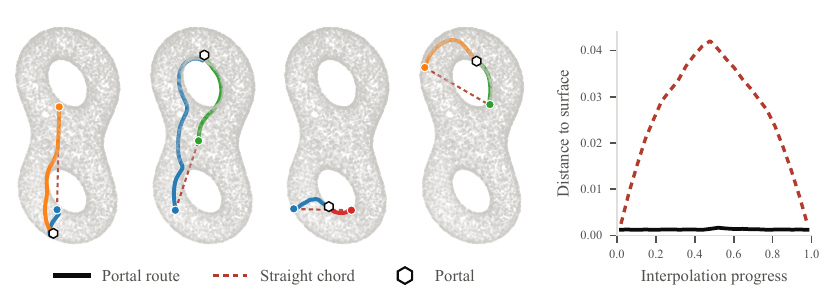}
    \caption{
        Interpolation trajectories between two charts through a portal on the double-torus geometry.
        For any two points on the torus, we can create interpolation trajectories in our charts' latent spaces that remain close to the torus surface (the data manifold).
        The red line illustrates straight-line interpolation through a global latent space, which leaves the surface and thus passes through implausible points.
        Optimizing portal placement to select shorter paths is left as future work.
    }
    \label{fig:double-torus-interpolation-trajectories}
\end{figure}

\paragraph{Generated samples.}

\autoref{fig:samples} includes additional samples generated from each of our charts and from a capacity-matched vanilla VAE baseline.
Samples are drawn from each chart's prior at random. Sampling from our VAA model, only requires the sampled chart's weights to be active, resulting in faster sampling compared to our baselines which need all the model parameters.

\begin{figure}[p]
    \centering
    \begin{subfigure}[b]{\textwidth}
        \centering
        \includegraphics[width=\linewidth]{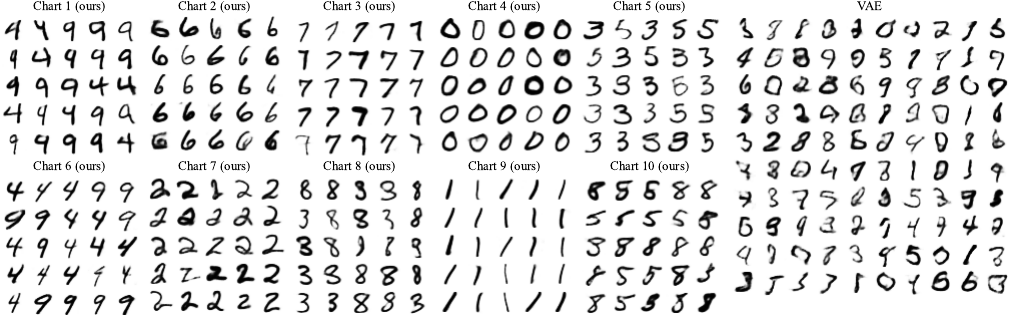}
        \subcaption{MNIST: 10 charts}
        \label{fig:mnist-samples}
    \end{subfigure}
    \begin{subfigure}[b]{\textwidth}
        \centering
        \includegraphics[width=\linewidth]{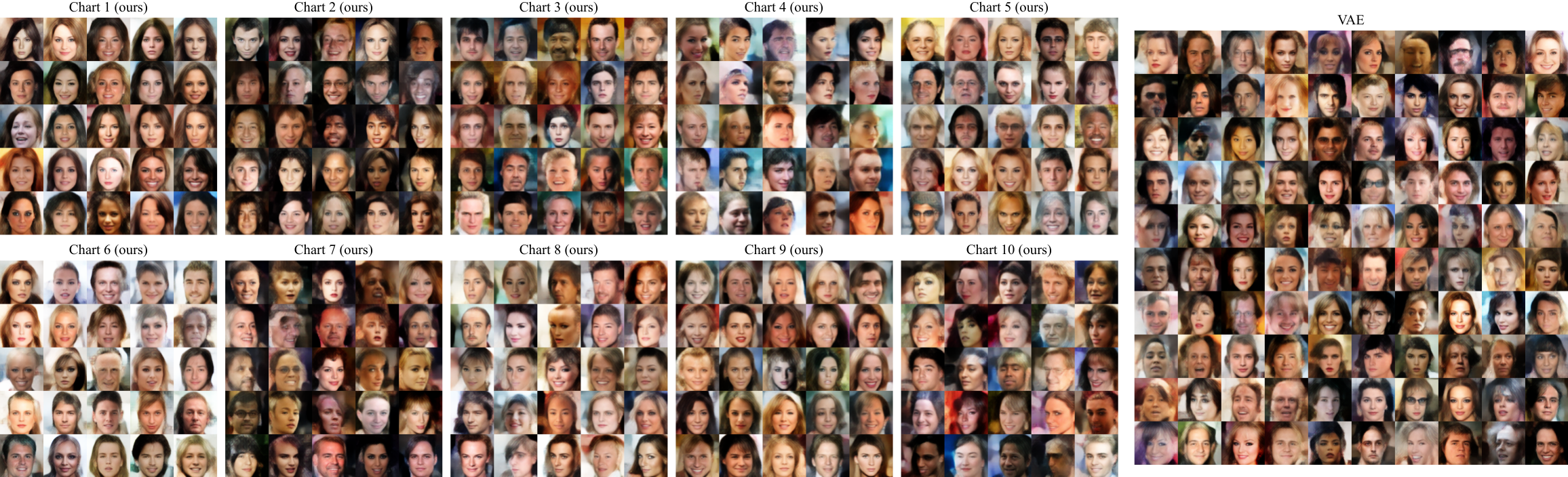}
        \subcaption{CelebA: 10 charts}
        \label{fig:celeba-samples}
    \end{subfigure}
     \begin{subfigure}[b]{\textwidth}
        \centering
        \includegraphics[width=\linewidth]{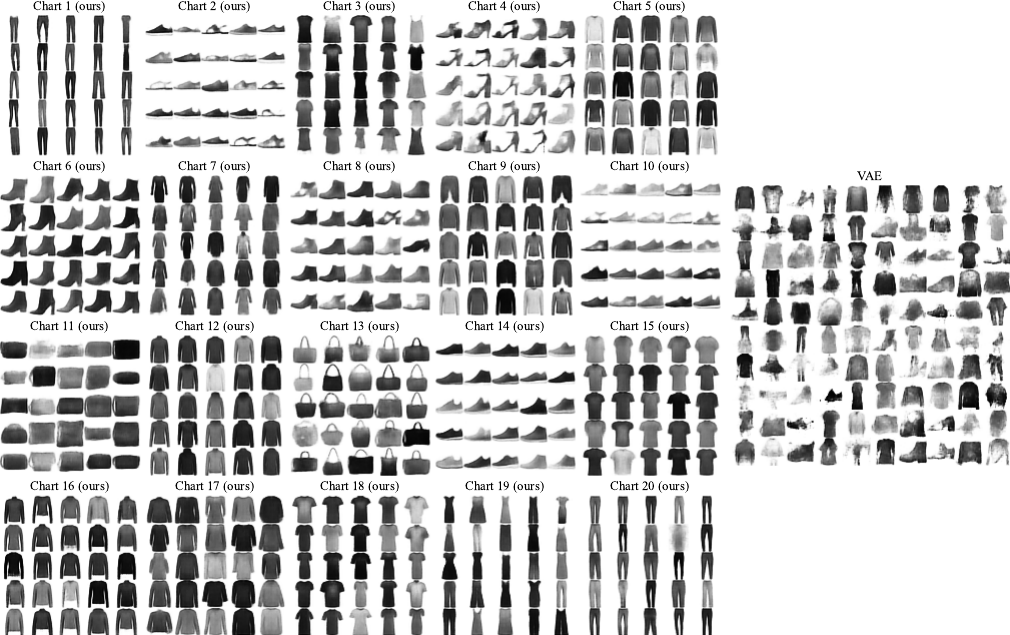}
        \subcaption{Fashion-MNIST: 20 charts}
        \label{fig:fashion-samples}
    \end{subfigure}
    \caption{
        Sample images drawn from 10-, 10-, and 20-chart VAAs trained on MNIST, CelebA, and Fashion-MNIST, respectively.
        Each grid represents samples generated from a distinct chart.
        Samples from a capacity-matched VAE are shown on the right.
        \label{fig:samples}
    }
\end{figure}

\paragraph{Geometric interpolation results.}

\autoref{fig:double-torus-interpolation-trajectories} shows interpolation trajectories over the double torus, including paths across chart boundaries through portals.

\end{document}